%% file: main_theory_revised.tex
\pdfoutput=1
\documentclass[10pt]{article}
\usepackage[preprint]{tmlr}

\input{math_commands.tex}

\usepackage{amsmath}
\usepackage{mathtools}
\usepackage{amsthm}
\usepackage{amssymb}

\usepackage{graphicx}
\usepackage{booktabs}
\usepackage{caption}
\usepackage{algorithm}
\usepackage{algpseudocode}
\usepackage[dvipsnames]{xcolor}

\usepackage{tikz}
\usetikzlibrary{positioning,fit,calc,arrows.meta}

\usepackage{url}
\usepackage{hyperref}
\usepackage{cleveref}

\newtheorem{theorem}{Theorem}[section]
\newtheorem{lemma}[theorem]{Lemma}
\newtheorem{proposition}[theorem]{Proposition}
\newtheorem{corollary}[theorem]{Corollary}

\title{FedPCD: Priority-Constrained Federated Aggregation}

    \author{\name Yousef Irshaid\textsuperscript{1,\,*} \email b00093447@aus.edu \\
      \name Dara Varam\textsuperscript{1,\,2,\,*} \email b00081313@aus.edu \\
      \name Mohamed I. AlHajri\textsuperscript{1} \email mialhajri@aus.edu \\[0.06in]
      \addr $^1$ Department of Computer Science \& Engineering, American University of Sharjah \\
      $^2$ Senseable City Laboratory, Massachusetts Institute of Technology \\
      $^{*}$ Equal contribution.
      }

\def\month{MM}
\def\year{YYYY}
\def\openreview{\url{https://openreview.net/forum?id=XXXX}}

\begin{document}

\maketitle

\begin{abstract}
Federated learning (FL) seeks to optimize a shared global model from decentralized client data. While effective for average performance, this process can produce unfair performance across clients, particularly under statistically heterogeneous data distributions. Existing fairness-aware and multi-objective FL methods address this through client re-weighting, gradient correction, or common-descent formulations, but often change the global optimization direction or treat client objectives symmetrically. We propose FedPCD, a priority-constrained federated optimization framework that separates the global optimization objective from client protection. FedPCD preserves the conventional federated aggregate as the primary objective, while allowing only the minimum correction necessary to meet client performance constraints. Across heterogeneous FL benchmarks, FedPCD consistently outperforms state-of-the-art baselines, achieving the highest global accuracy across all configurations and leading worst-10\% client accuracy in nearly all. FedPCD further preserves competitive trade-offs between global accuracy and client performance variability. Beyond this trade-off, FedPCD demonstrates faster empirical convergence, reaching matching or superior performance in substantially fewer communication rounds.

\end{abstract}

\section{Introduction}

Federated learning (FL) enables multiple clients to collaboratively train a shared model while keeping their data decentralized \citep{mcmahanCommunicationEfficientLearningDeep2017, kairouzAdvancesOpenProblems2021}. Standard federated averaging (FedAvg) constructs a global update by aggregating the locally trained model updates of participating clients \citep{mcmahanCommunicationEfficientLearningDeep2017}. This aggregation provides a simple and effective way to optimize a shared objective. However, client data distributions are commonly heterogeneous \citep{liFederatedOptimizationHeterogeneous2020, karimireddySCAFFOLDStochasticControlled2020}, and when local data distributions differ, the resulting direction can make substantial progress for some clients while providing little, none or even adverse progress for others \citep{wangFederatedLearningFair2021, huFederatedLearningMeets2022, panFedMDFGFederatedLearning2023}. This motivates performance fairness in FL, in which the goal is to improve the uniformity of client-level performance while maintaining the utility of the global model \citep{liFairResourceAllocation2019, shiFairnessAwareFederatedLearning2024}

A substantial body of work has consequently modified the optimization objective to improve client-level fairness \cite{shiFairnessAwareFederatedLearning2024}. Among the earliest works, Agnostic FL (AFL) formulates a minimax objective over mixtures of client distributions and optimizes it by adaptively updating the mixture weights according to their losses \citep{mohriAgnosticFederatedLearning2019}. Similarly, q-Fair FL (q-FFL) modifies the objective to place greater emphasis on clients with higher losses \citep{liFairResourceAllocation2019}. Other approaches formulate federated learning as a multi-objective optimization problem, treating client losses as separate objectives and seeking an update direction that provides simultaneous progress across clients \citep{huFederatedLearningMeets2022, panFedMDFGFederatedLearning2023, panFedLFLayerWiseFair2024}. Gradient-conflict-based methods similarly modify the aggregated direction when client updates are incompatible \citep{wangFederatedLearningFair2021}. Despite their different formulations, these methods share the broader goal of modifying the global optimization process to account for client-level objectives rather than aggregate performance alone.

Multi-objective approaches are of particular relevance to our approach because they treat client losses as separate objectives and seek a common descent direction that improves the objectives of all participating clients. These methods typically seek convergence to a Pareto-stationary solution, at which no direction exists that can simultaneously decrease all client objectives to first order \citep{huFederatedLearningMeets2022}. While this provides a principled way to account for client-level progress, the client objectives are treated symmetrically when constructing the global update. In standard FL, however, the FedAvg direction already represents the update preferred by the shared global objective. Replacing this direction may be unnecessary when it already provides sufficient progress for the individual client objectives.

This idea motivates FedPCD, a federated adaptation of Priority-Constrained Descent (PCD) \citep{varamNotAllObjectives2026}. PCD is a hierarchical multi-objective optimization framework that distinguishes a primary objective from secondary objectives and seeks the smallest deviation from the primary descent direction required to guarantee progress on secondary objectives. In FedPCD, the FedAvg update assumes the primary role, while individual client update directions define secondary constraints. Therefore, the server seeks a direction that remains as close as possible to the FedAvg update while satisfying the prescribed client-level requirements. If the FedAvg direction already satisfies these requirements, FedPCD leaves it unchanged. Otherwise, it introduces the minimum correction necessary to satisfy the violated constraints. Consequently, improved uniformity of client performance arises from the client-level protection, rather than optimizing an explicit fairness surrogate. The contributions of this paper are as follows:

\begin{itemize}
     \item We introduce FedPCD, a priority-constrained formulation that treats the conventional federated aggregate as the primary objective and individual client updates as secondary constraints, allowing only for the minimum deviation required to satisfy these constraints.

     \item We extend the formulation to partial participation through a client-history mechanism that retains recently observed client directions to increase representation.

     \item We evaluate FedPCD across several FL settings, including 2 levels of participation, 3 datasets, and 3 data heterogeneity levels. FedPCD consistently achieves the highest global accuracy, superior worst-$10\%$ client accuracy, and competitive client-performance spread compared to state-of-the-art fairness-aware FL methods, indicating that preserving client-level progress naturally translates into strong tail-client and fairness performance.

    \item We demonstrate how FedPCD substantially improves empirical communication efficiency, reaching FedAvg's final accuracy using only $10\%$--$44\%$ of the total communication rounds and outperforming the strongest baseline in convergence speed in 16 of 18 configurations.

    \item We establish an $\mathcal{O}(1/T)$ convergence rate for the round-averaged squared $\ell_2$-norm of the global gradient under deterministic full participation, and show convergence to simultaneous client stationarity under the stated geometric conditions.

\end{itemize}

\section{Related Work}
\label{related_works}

Existing approaches to client fairness broadly fall into 3 paradigms: reweighting client objectives, resolving gradient conflicts, or formulating fair aggregation as a multi-objective optimization problem (Fig.~\ref{fig:taxonomy}). We examine these approaches below and contrast them with the asymmetric formulation of FedPCD.

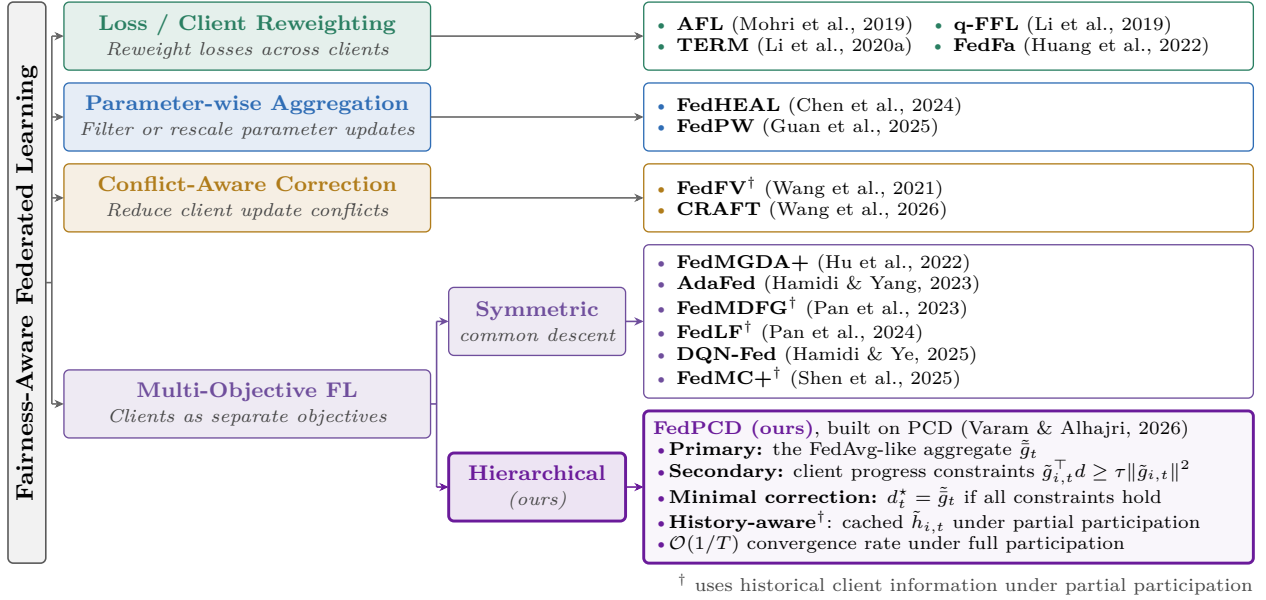
\begin{figure}[t]
\centering
\input{figures/fig_taxonomy.tex}
\caption{A taxonomy of the fairness-aware federated learning methods discussed in Sec.~\ref{related_works}.}
\label{fig:taxonomy}
\end{figure}

\subsection{Fair Federated Learning}
Performance fairness in FL seeks to reduce disparities in the performance of a shared global model across clients \cite{shiFairnessAwareFederatedLearning2024}. A major line of work addresses this problem by changing the relative influence of client losses during optimization to shift the training away from average performance. These works range from worst-case optimization over mixtures of client distributions \citep{mohriAgnosticFederatedLearning2019} to emphasizing disadvantaged clients in the weighted objective \citep{liFairResourceAllocation2019, liTiltedEmpiricalRisk2020, huangFairnessAccuracyHorizontal2022}. More recent methods move beyond a single client-level weight by also modifying parameter contributions. FedHEAL suppresses less informative updates and adjusts aggregation according to domain diversity, while FedPW combines parameter selection with a weighting strategy dependent on the training phase to balance a tradeoff between generalization and fairness \citep{chenFairFederatedLearning2024, guanRethinkingFairFederated2025}.

Another line of work connects performance disparities to conflicts among client updates. FedFV identifies conflicting gradients as a source of unfairness and modifies client directions through iterative projection before aggregation \citep{wangFederatedLearningFair2021}. More recently, CRAFT formulates conflict resolution as a constrained correction problem, seeking an update close to a reference direction while enforcing positive alignment with the participating clients \citep{wangCRAFTConflictResolvedAggregation2026}. Furthermore, partial participation complicates fairness-aware learning because the updates available in a communication round represent only a subset of the population. Therefore, historical information has been used to account for absent clients in FedFV, FedMDFG, FedLF, and FedMC+ \citep{wangFederatedLearningFair2021, panFedMDFGFederatedLearning2023, panFedLFLayerWiseFair2024, shenMultiobjectiveFederatedLearning2025}. These approaches motivate the use of historical client information to improve representation under partial participation while exposing a trade-off between broader client coverage and gradient staleness \citep{panFedLFLayerWiseFair2024}.

\subsection{Multi-Objective Federated Optimization}
Beyond reweighting and conflict-aware aggregation, several methods approach fair aggregation as a multi-objective optimization (MOO) problem. Building on the Multiple Gradient Descent Algorithm (MGDA)~\citep{desideriMultiplegradientDescentAlgorithm2012}, FedMGDA+ seeks Pareto-stationary solutions by constructing a common descent direction across all client gradients~\citep{huFederatedLearningMeets2022}. Subsequent works refine this common-descent paradigm: AdaFed scales updates by client loss magnitudes to prioritize lagging participants~\citep{hamidiAdaFedFairFederated2023}, FedMDFG and FedLF incorporate explicit fairness regularizers, with FedLF operating layer-wise~\citep{panFedMDFGFederatedLearning2023, panFedLFLayerWiseFair2024}, and DQN-Fed extends the formulation to quasi-Newton directions~\citep{hamidiDistributedQuasiNewtonMethod2025}. While these methods treat client losses symmetrically to find a shared compromise, FedMC+~\citep{shenMultiobjectiveFederatedLearning2025} uniquely separates the average objective from local conflict mitigation. FedPCD shares this distinction between aggregate and local objectives, but departs fundamentally in execution. Rather than solving a joint multi-objective problem, FedPCD treats client objectives as hard minimum-progress secondary constraints, projecting the standard federated update only by the minimal margin needed to satisfy them.

\section{The Federated PCD Framework}
\label{methodology}

\subsection{Preliminaries}
\label{preliminaries}

\subsubsection{Federated Learning Setup}
\label{sec:federated_setup}
Consider a federation of $N$ clients. Client $i$ holds a local data distribution and optimizes a local objective $F_i(\theta)$. The standard federated objective is the weighted average of the client objectives,
\begin{equation}
    F(\theta)
    =
    \sum_{i=1}^{N} p_i F_i(\theta),
    \qquad
    p_i \geq 0,
    \qquad
    \sum_{i=1}^{N} p_i = 1.
    \label{eq:global_objective}
\end{equation}
The aggregation weights are commonly chosen in proportion to the
size of each client's local dataset,
$p_i = \frac{|\mathcal{D}_i|}{\sum_{j=1}^{N}|\mathcal{D}_j|},$
where $\mathcal{D}_i$ denotes the local training dataset of client $i$. FedAvg approximately optimizes~Equation \ref{eq:global_objective} through repeated local training and server aggregation \citep{mcmahanCommunicationEfficientLearningDeep2017}. At communication round $t$, the server broadcasts the current global model $\theta_t$ to a subset of clients $\mathcal{S}_t$. We denote the client participation rate by $\rho\in(0, 1]$, so that the number of participating clients per round is $\rho N$. Each participating client initializes its local model as $\theta_{i,t}^{(0)} = \theta_t$ and performs $K$ local optimization steps. For local SGD, a step can be written as
\begin{equation}
    \theta_{i,t}^{(k+1)}
    =
    \theta_{i,t}^{(k)}
    -
    \eta_{\ell}
    \widehat{\nabla} F_i
    \left(\theta_{i,t}^{(k)}\right),
    \qquad
    k=0,\ldots,K-1,
    \label{eq:local_training}
\end{equation}

where $\eta_{\ell}$ is the local learning rate and
$\widehat{\nabla}F_i$ denotes a stochastic estimate of the local
gradient. After local training, client $i$ returns the client update $\Delta_{i,t} = \theta_t-\theta_{i,t}^{(K)}$.
FedAvg then aggregates the participating client updates as $g_t = \sum_{i\in\mathcal{S}_t} \bar p_{i,t}\Delta_{i,t}$ where
\begin{equation}
    g_t
    =
    \sum_{i\in\mathcal{S}_t}
    \bar p_{i,t}\Delta_{i,t},
    \qquad
    \bar p_{i,t}
    =
    \frac{p_i}
    {\sum_{j\in\mathcal{S}_t}p_j},
    \label{eq:fedavg_direction}
\end{equation}
and updates the global model according to
$
    \theta_{t+1}
    =
    \theta_t-\eta_g g_t,
    \label{eq:fedavg_update}
$
where $\eta_g$ is the server learning rate parameter that can be introduced to control the step size \citep{reddiAdaptiveFederatedOptimization2020}.

\subsubsection{Priority-Constrained Descent}

PCD addresses hierarchical multi-objective optimization problems by distinguishing a primary objective from a set of secondary objectives \citep{varamNotAllObjectives2026}. Let $\tilde g_0$ denote the normalized descent direction of the primary objective and let $\tilde g_j$, $j=1,\ldots,J$, denote the normalized descent directions of the secondary objectives. PCD formulates the update as a convex quadratic program (QP) that seeks a direction as close as possible to the primary objective while guaranteeing a prescribed amount of first-order progress on each secondary objective:

\begin{equation}
    \tilde d^\star
    =
    \arg\min_{\tilde d}
    \frac{1}{2}
    \left\lVert \tilde d-\tilde{g}_0\right\rVert^2
    \quad
    \mathrm{s.t.}
    \quad
    \tilde g_j^\top \tilde d
    \geq
    \tau\lVert \tilde g_j\rVert^2,
    \qquad
    j=1,\ldots,J,
    \label{eq:pcd_generic}
\end{equation}
The QP in Equation~\ref{eq:pcd_generic} has a strictly convex quadratic objective and linear inequality constraints. It preserves $\tilde g_0$ as the preferred descent direction, while the constraints determine whether that direction provides sufficient progress on the secondary objectives. Geometrically, each constraint defines a half-space of acceptable directions, and their intersection defines the feasible region of the PCD problem. Therefore, whenever the problem is feasible, PCD can be interpreted as projecting the primary objective onto this feasible region. Since the objective is strictly convex and the constraints are linear, this projection is unique. If $\tilde g_0$ already satisfies every secondary progress constraint, no correction is required and $\tilde d^\star = \tilde g_0$. Otherwise, the solution differs from the primary objective only through the secondary constraints that are active at the optimum. Let $\mathcal{A}$ denote this active set. The optimality conditions of Equation~\ref{eq:pcd_generic} give
\begin{equation}
    \tilde d^\star
    =
    \tilde g_0
    +
    \sum_{j\in\mathcal{A}}
    \mu_j^\star \tilde g_j,
    \qquad
    \mu_j^\star\geq0,
    \label{eq:pcd_active_form}
\end{equation}
where the multipliers $\mu_j^\star$ determine the amount of correction required along each active secondary direction. They satisfy the complementary-slackness conditions
\begin{equation}
    \mu_j^\star
    \left(
        \tilde g_j^\top \tilde d^\star
        -
        \tau\lVert \tilde g_j\rVert^2
    \right)
    =
    0,
    \qquad
    j=1,\ldots,J.
    \label{eq:pcd_complementarity}
\end{equation}
Hence, a secondary objective whose constraint is already satisfied does not directly modify the primary objective.

The parameter $\tau\in(0,1]$ controls the strength of these requirements. Larger values demand greater progress on the secondary objectives and can require a larger deviation from $\tilde g_0$. PCD thus creates an explicit hierarchy between the objectives: the primary objective determines the preferred direction of optimization, while the secondary objectives determine which deviations from that direction are necessary.

\subsection{Priority-Constrained Federated Aggregation}
\label{sec:fedpcd_aggregation}

\begin{figure*}[t]
    \centering
    \includegraphics[width=\linewidth]{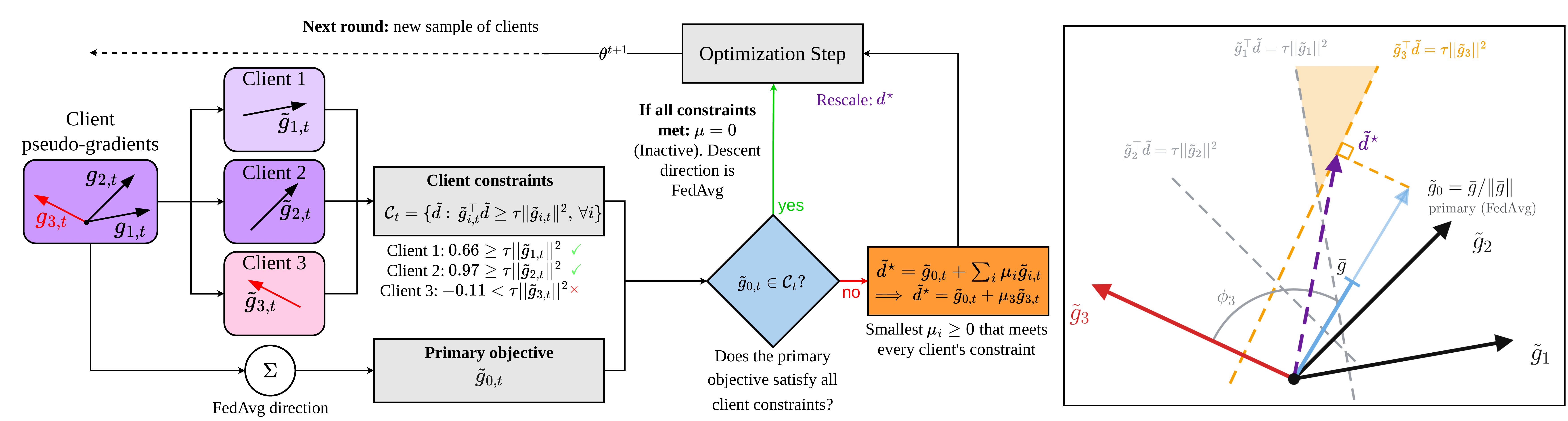}
    \caption{FedPCD system diagram, showing an example with three participating clients at communication round $t$. \textbf{Left:} Algorithmic steps taken to obtain a final direction in gradient space ($\tilde d^\star$). \textbf{Right:} Objective space view, showing the resultant $\tilde d^\star$. The half-space projection of each gradient is used to find the feasible region, which then allows us to see the minimum distortion required to satisfy all secondary constraints.}
    \label{fig:system-diagram}
\end{figure*}

FedPCD transfers the hierarchical structure of PCD to federated optimization by treating the conventional federated aggregate as the primary objective and the individual client updates as secondary directions. Therefore, the server preserves the aggregate update as its preferred direction while imposing client-level progress requirements through the secondary constraints. A system diagram is provided in Fig.~\ref{fig:system-diagram}.

The standard FedAvg formulation in Section~\ref{sec:federated_setup} operates on the model displacement $\Delta_{i,t}$. FedPCD converts this displacement to a local pseudo-gradient by dividing by the local learning rate:
\begin{equation}
    g_{i,t}
    =
    \frac{\Delta_{i,t}}{\eta_\ell}
    =
    \frac{
        \theta_t-\theta_{i,t}^{(K)}
    }{
        \eta_\ell
    }=\sum_{k=0}^{K-1} \widehat{\nabla}F_i(\theta_{i,t}^{(k)}).
    \label{eq:fedpcd_client_direction}
\end{equation}

The primary FedPCD direction is formed by aggregating the participating client pseudo-gradients using the same weights as FedAvg to create a FedAvg-like descent direction:
\begin{equation}
    \bar g_t
    =
    \sum_{i\in\mathcal{S}_t}
    \bar p_{i,t}g_{i,t}
    =
    \frac{1}{\eta_\ell}
    \sum_{i\in\mathcal{S}_t}
    \bar p_{i,t}\Delta_{i,t}
    =
    \frac{g_t}{\eta_\ell},
    \label{eq:fedpcd_aggregate}
\end{equation}

where $g_t$ is the conventional FedAvg aggregate from Equation~\ref{eq:fedavg_direction}. Since $\eta_\ell>0$, this transformation preserves the direction of the conventional federated aggregate while expressing it on the same pseudo-gradient scale as the client directions. FedPCD then constructs the constrained problem from the normalized aggregate and the normalized participating client directions where
\begin{equation}
    \tilde g_{0,t}
    =
    \frac{\bar g_t}
    {\lVert\bar g_t\rVert+\varepsilon},
    \qquad
    \tilde g_{i,t}
    =
    \frac{g_{i,t}}
    {\lVert g_{i,t}\rVert+\varepsilon},
    \label{eq:fedpcd_normalization}
\end{equation}
is the standard $\ell_2$-normalization with a small numerical constant $\varepsilon>0$ . Although the original PCD work reports stronger empirical performance with EMA normalization than with \(\ell_2\)-normalization \citep{varamNotAllObjectives2026}, directly maintaining EMA statistics for client directions is less straightforward under partial participation because each client's update is only observed when that client is sampled. For infrequently participating clients, these normalization statistics can therefore become stale. We examine the effect of EMA normalization separately in App.~\ref{app:ablations}. The current federated aggregate provides the primary objective, while each participating client contributes a secondary progress requirement. Therefore, the resulting FedPCD program is
\begin{equation}
    \tilde d_t^\star
    =
    \arg\min_{\tilde d}
    \frac{1}{2}
    \left\lVert
        \tilde d-\tilde g_{0,t}
    \right\rVert^2
    \quad
    \mathrm{s.t.}
    \quad
    \tilde g_{i,t}^{\top}\tilde d
    \geq
    \tau \lVert \tilde g_{i,t}\rVert^2,
    \qquad
    i\in\mathcal{S}_t.
    \label{eq:fedpcd_qp}
\end{equation}

The solution $\tilde d_t^\star$ determines the direction that remains closest to the current federated aggregate while satisfying the client-level progress requirements. However, since the QP is constructed using normalized directions, the magnitude of $\tilde d_t^\star$ no longer reflects the magnitude of the current federated update. Therefore, FedPCD rescales the solution using the norm of the current aggregate,
\begin{equation}
    d_t
    =
    \frac{\lVert \bar g_t\rVert}
         {\lVert \tilde d_t^\star\rVert}
    \tilde d_t^\star,
    \label{eq:fedpcd_rescaling}
\end{equation}
and applies the resulting direction through the server update
\begin{equation}
    \theta_{t+1}
    =
    \theta_t-\eta_g d_t.
    \label{eq:fedpcd_global_update}
\end{equation}

Thus, the PCD program determines how the direction is modified, while the magnitude of the update remains tied to the federated aggregate. An important consequence is that FedPCD introduces no directional correction when the current aggregate already satisfies all the constraints. In this case,  $\tilde d_t^\star=\tilde g_{0,t}$, and the rescaling in Equation~\ref{eq:fedpcd_rescaling}  gives $d_t=\bar g_t$, preserving the conventional federated direction.

\subsection{Partial Participation and Client History}

Under partial participation, Equation~\ref{eq:fedpcd_qp} contains constraints only for clients sampled in the current round. This provides only a narrow representation of the client population, especially when the participation rate is small and data heterogeneity across clients is severe. Therefore, FedPCD incorporates a history window $H\ge0$ of communication rounds that stores the pseudo-gradient of each client that participated within this window, representing additional constraints to the QP in Equation~\ref{eq:fedpcd_qp}. Each client contributes only their most recent update, replacing their previously stored direction if they participate again within the history window. If a client has not participated within the $H$ rounds, its cached pseudo-gradient is removed. This way, the current participants contribute their latest directions, while recently absent clients continue to influence the federated aggregate through their cached updates for a specific history window $H$. Importantly, client history is used only to construct the secondary constraints. The primary direction $\bar g_t$ remains the current federated aggregate and is computed exclusively from clients participating in round $t$. For a cached client direction $h_{i,t}$, the corresponding FedPCD constraint is $\tilde h_{i,t}^\top \tilde d \ge \tau \| \tilde h_{i,t}\|^2$, where $\tilde h_{i,t}$ denotes the normalized cached pseudo-gradient. Setting $H=0$ reduces FedPCD to the current-round formulation in Equation~\ref{eq:fedpcd_qp}. The complete FedPCD framework is summarized in Algorithm \ref{alg:fedpcd}.

Under independent uniform participation, the per-round client participation probability is $\rho$ and the expected number of distinct clients represented after the history window is populated becomes $N(1-(1-\rho)^{H+1})$. Therefore, the impact of the history window represents a tradeoff between broader client representation at low participation rates and the potential increase in staleness of updates for the absent clients \cite{panFedLFLayerWiseFair2024}. Fig.~\ref{fig:history-demo} illustrates the idea of history caching for FedPCD in a synthetic setup. We simulate a scenario with $N = 30$ clients with $10\%$ participation and $\tau = 0.1$. Tracing 30 rounds of this process, we see that without history, the program never contains more than 3 constraints (or secondary objectives), skewing the resultant $\tilde d^\star$ towards the current-round gradients. With history, the cached minority clients still constrain the program and incorporate them into the optimization step $\tilde d^\star$. Since there is no backlog of gradients in the early rounds, we do not see the impact of history, but the separation becomes clear at round 16, where $\tilde d^\star$ with and without history point in different directions.

\begin{figure*}[!t]
    \centering
    \includegraphics[width=0.9\linewidth]{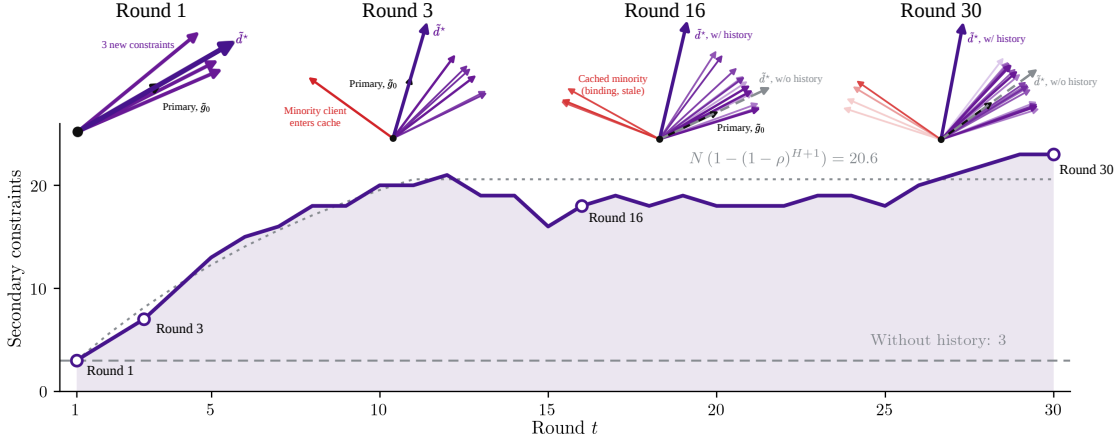}
    \caption{Demonstration of history for FedPCD, with $N = 30$ clients, $10\%$ participation and $H = 10$. \textbf{Top:} PCD program at 4 selected rounds, marked on the curve. Paler arrows are older cache entries, and red arrows are minority clients that do not satisfy the PCD condition. \textbf{Bottom:} number of secondary constraints per round, which grows from 3 participants as a function of the communication rounds.}
    \label{fig:history-demo}
\end{figure*}

\begin{algorithm}[t]
\caption{FedPCD with Client History}
\label{alg:fedpcd}
\begin{algorithmic}[1]
\Require Initial model $\theta_0$, rounds $T$, participation rate $\rho$,
local steps $K$, learning rates $\eta_\ell,\eta_g$, constraint parameter $\tau$,
history window $H$, aggregation weights $\{p_i\}_{i=1}^{N}$

\State Initialize empty history buffer $\mathcal{B}$

\For{$t=0,\ldots,T-1$}

    \State Sample $\mathcal{S}_t$ according to participation rate $\rho$
    and broadcast $\theta_t$

    \ForAll{$i\in\mathcal{S}_t$ \textbf{in parallel}}
        \State $\theta_{i,t}^{(0)}\gets\theta_t$
        \For{$k=0,\ldots,K-1$}
            \State
            $\theta_{i,t}^{(k+1)}
            \gets
            \theta_{i,t}^{(k)}
            -
            \eta_\ell
            \widehat{\nabla}F_i(\theta_{i,t}^{(k)})$
        \EndFor
        \State Send
        $\Delta_{i,t}\gets\theta_t-\theta_{i,t}^{(K)}$
        to the server
    \EndFor

    \State Compute
    $g_{i,t}\gets\Delta_{i,t}/\eta_\ell$,
    $\forall i\in\mathcal{S}_t$

    \State Compute
    $\bar g_t
    \gets
    \sum_{i\in\mathcal{S}_t}
    \bar p_{i,t}g_{i,t}$
    \State Remove entries from $\mathcal{B}$ older than $H$ rounds
    \State Set
    $h_{i,t}\gets\mathcal{B}[i]$
    $\forall \, i\in\mathcal{B}$

    \State Normalize $\bar g_t$, $\{g_{i,t}\}$, and $\{h_{i,t}\}$
    using Equation~\ref{eq:fedpcd_normalization}

    \State Solve $\displaystyle
        \tilde d_t^\star =
        \arg\min_{\tilde d} \frac{1}{2}\lVert \tilde d-\tilde g_{0,t}\rVert^2
        \quad \mathrm{s.t.}\quad
        \tilde g_{i,t}^{\top}\tilde d \geq \tau\lVert\tilde g_{i,t}\rVert^2
        \ \forall i\in\mathcal{S}_t,\quad
        \tilde h_{j,t}^{\top}\tilde d \geq \tau\lVert\tilde h_{j,t}\rVert^2
        \ \forall j\in\mathcal{B}
        $

    \State Rescale
    $d_t
    \gets
    \dfrac{\lVert\bar g_t\rVert}
          {\lVert \tilde d_t^\star\rVert}
    \tilde d_t^\star$

    \State Update
    $\theta_{t+1}
    \gets
    \theta_t-\eta_g d_t$

    \State Replace any entry in $\mathcal{B}$ for clients in $\mathcal{S}_t$ with their current pseudo-gradient.
\EndFor

\State \Return $\theta_T$
\end{algorithmic}
\end{algorithm}

\section{Theoretical Analysis}
\label{sec:theory}
To understand the theoretical implications of enforcing the client-level secondary progress constraints, we analyze the convergence behavior of the resulting update with respect to both the global federated objective and the individual client objectives. The analysis follows a standard FL setting with full client participation and a common number $K$ of deterministic (full-batch) local gradient steps, and sets $H=0$. We further assume that each $F_i$ is $L$-smooth and that the global objective $F$ is bounded below by $F_\star$. These are standard assumptions in federated optimization analyses \citep{karimireddySCAFFOLDStochasticControlled2020, wangCRAFTConflictResolvedAggregation2026}. The additional FedPCD-specific condition concerns the geometry of the constrained direction. Let $\mathcal I_t=\{i:g_{i,t}\neq0\}$ denote the clients with nonzero pseudo-gradients. We use $\gamma>0$ to denote a uniform lower bound on the directional alignment produced by the constrained program,
\begin{equation}
    \frac{
        g_{i,t}^{\top}\tilde d_t^\star
    }{
        \lVert g_{i,t}\rVert
        \lVert \tilde d_t^\star\rVert
    }
    \geq \gamma,
    \qquad i\in\mathcal I_t.
    \label{eq:main_alignment_margin}
\end{equation}
The constant $\gamma$ is not an additional hyperparameter. It quantifies the smallest cosine alignment induced by the constrained program over the analyzed rounds and is distinct from the progress threshold $\tau$. We assume that the QP remains feasible and that this margin holds uniformly throughout training. When $K>1$, later local gradients are evaluated at locally updated models rather than at $\theta_t$. Their resulting deviation from $K\nabla F_i(\theta_t)$ is captured by
\begin{equation}
    r
    =
    \frac{1}{K}
    \sum_{k=0}^{K-1}
    \left[(1+L\eta_\ell)^k-1\right],
    \qquad
    \gamma_{\mathrm{eff}}=\gamma(1-r)-r,
    \label{eq:main_drift_factors}
\end{equation}
where $\gamma_{\mathrm{eff}}$ is the remaining client-alignment margin after accounting for local drift. Complete assumptions and proofs are given in App.~\ref{app:theory}.

\begin{theorem}[Finite-time convergence of FedPCD]
\label{thm:fedpcd_stationarity}
Suppose the conditions above hold. If $\gamma_{\mathrm{eff}}>0$ and
\begin{equation}
    0<\eta_g
    \leq
    \frac{\gamma_{\mathrm{eff}}}{LK(1+r)},
    \label{eq:fedpcd_server_stepsize}
\end{equation}
then, after $T$ communication rounds,
\begin{equation}
    \frac{1}{T}
    \sum_{t=0}^{T-1}
    \left\lVert\nabla F(\theta_t)\right\rVert^2
    \leq
    \frac{
        2\left[F(\theta_0)-F_\star\right]
    }{
        \eta_g \gamma_{\mathrm{eff}}\gamma K(1-r)T
    }
    =
    \mathcal{O}(1/T).
    \label{eq:fedpcd_convergence}
\end{equation}
\end{theorem}

Theorem~\ref{thm:fedpcd_stationarity} gives an $\mathcal{O}(1/T)$ deterministic nonconvex stationarity rate with fixed admissible local and server learning rates, including for $K>1$.

\begin{corollary}[Composite multi-stationarity]
\label{cor:fedpcd_cms}
Under the conditions of Theorem~\ref{thm:fedpcd_stationarity},
\begin{equation}
    \frac{1}{T}
    \sum_{t=0}^{T-1}
    \sum_{i=1}^{N}
    p_i
    \left\lVert\nabla F_i(\theta_t)\right\rVert^2
    =
    \mathcal{O}(1/T).
    \label{eq:fedpcd_client_stationarity}
\end{equation}
For an infinite run satisfying the conditions uniformly, every client gradient converges to zero and every accumulation point is composite multi-stationary (CMS).
\end{corollary}

CMS requires every client objective to be individually stationary \citep{varamNotAllObjectives2026}, and is therefore stronger than Pareto stationarity, which may occur through cancellation among nonzero client gradients. The CMS conclusion is conditional on the stated geometric conditions, since arbitrary heterogeneous client objectives need not admit a common stationary point. The complete proofs, the scope of the client-progress and CMS guarantees, and a comparison with CRAFT are provided in App.~\ref{app:theory}.

\section{Experiments}
\subsection{Experimental Setup}
\paragraph{Datasets and Federated Setup.} We evaluated FedPCD on CIFAR-10, CIFAR-100, and Fashion-MNIST datasets \cite{krizhevsky2009learning, xiaoFashionMNISTNovelImage2017}.  We use a CNN with 3 convolutional layers for CIFAR-10, Resnet18 for CIFAR-100, and a multilayer perceptron (MLP) with two hidden layers for FashionMNIST. To simulate non-IID data, we use Dirichlet partitioning following \citet{acarFederatedLearningBased2020} with $\alpha \in \{0.01, 0.1, 1\}$, where smaller values of $\alpha$ correspond to stronger statistical heterogeneity across clients. We consider 5\% and 10\% participation rates of clients, corresponding to $\rho\in\{0.05, 0.1\}$, in a pool of $N=100$ clients for each dataset. Selected clients perform one epoch of local training before returning the update to the server. We train the CIFAR-10 and CIFAR-100 dataset models for 500 communication rounds, and FashionMNIST for 300 rounds. For FedPCD-specific hyperparameters in the main results, we experiment with $H=10$ and $\tau \in \{0.005, 0.01, 0.02, 0.05, 0.075, 0.1\}$. Additional model and training details are provided in App.~\ref{app:reproducibility}.

\paragraph{Metrics and Baselines.} Performance is evaluated using the mean global accuracy, the standard deviation, and the mean worst-10\% client accuracy as a measure of tail performance, averaged across 5 seeds. We compare FedPCD against methods representing several families of federated optimization and fairness approaches. FedAvg~\citep{mcmahanCommunicationEfficientLearningDeep2017} serves as the standard aggregation baseline. We include traditional fairness-aware methods AFL~\citep{mohriAgnosticFederatedLearning2019}, q-FFL~\citep{liFairResourceAllocation2019}, and Ditto~\citep{liDittoFairRobust2021}. We further compare against multi-objective and conflict-aware approaches FedMGDA+~\citep{huFederatedLearningMeets2022}, FedFV~\citep{wangFederatedLearningFair2021}, FedLF~\citep{panFedLFLayerWiseFair2024}, and CRAFT~\citep{wangCRAFTConflictResolvedAggregation2026}, along with more recent fairness-aware aggregation methods FedHEAL~\citep{chenFairFederatedLearning2024} and FedPW~\citep{guanRethinkingFairFederated2025}. For methods with tunable hyperparameters, we perform a hyperparameter sweep and use the best performing variant per-dataset. All experiments use the same 5 seeds that correspond to the same data splits for a fair comparison between the baselines. The detailed set of hyperparameter settings is included in App.~\ref{app:reproducibility}.

\subsection{Main Results}
\label{sec:main-results}

Table~\ref{tab:main_results} reports the global accuracy $\bar{a}$, the standard deviation across client accuracies $\sigma_c$, and the worst-$10\%$ client accuracy $\mathrm{W}_{10\%}$ on CIFAR-10, CIFAR-100, and Fashion-MNIST, at $5\%$ and $10\%$ participation and for $\alpha \in \{0.01, 0.1, 1\}$. The table lists every tested $\tau$, while figures use $\tau = 0.02$ and $H=10$ with $\alpha = 0.01$, which is the most heterogeneous regime we test. App.~\ref{app:all-alpha} reports the corresponding figures for the remaining $\alpha$ values. We organize the discussion around three central questions: (1) Is FedPCD more accurate on average? (2) Do the gains of FedPCD reach underrepresented clients? and (3) How quickly does FedPCD converge compared to SOTA baselines?

\subsubsection{Global and Tail Accuracy}
\label{sec:res-accuracy}

FedPCD achieves the best global accuracy in all 18 configurations of Table~\ref{tab:main_results}, and the highest worst-$10\%$ accuracy in 17 of them. The strongest baseline changes with the dataset and with the heterogeneity level, so we discuss the three datasets in turn.

\paragraph{CIFAR-10.} FedLF is the closest competitor at $\alpha = 0.01$, while CRAFT and FedFV are closest at $\alpha = 1$. The margin of FedPCD is largest where the data is most heterogeneous and participation is lowest. At $\alpha = 0.01$ and $5\%$ participation, FedPCD reaches $71.90\%$ global accuracy against $65.59\%$ for FedLF, and its worst-$10\%$ accuracy of $52.88\%$ exceeds that of FedLF by more than $14\%$. Raising participation to $10\%$ narrows the gap in global accuracy to $0.22\%$, but FedPCD keeps the higher global and worst-$10\%$ accuracy and matches FedLF on $\sigma_c$. At $\alpha = 1$, FedPCD leads the strongest baselines by $0.25\%$ and $0.78\%$ at $5\%$ and $10\%$ participation, respectively.

\paragraph{CIFAR-100.} Here, CRAFT is the strongest competitor, for which FedPCD improves the global accuracy by over $1.7\%$ to $5.6\%$ and the worst-$10\%$ accuracy by $0.8\%$ to $5.2\%$ points, with the improvement growing with $\alpha$. At $\alpha = 1$, FedPCD reaches $35.40\%$ and $36.47\%$ global accuracy at $5\%$ and $10\%$ participation, against $29.77\%$ and $31.29\%$ for CRAFT. At $\alpha = 0.01$, the task is hard for every method, and the worst-$10\%$ accuracy stays below $8\%$ for all of them.

\paragraph{Fashion-MNIST.} We observe smaller margins here than on the CIFAR datasets, between $0.6\%$ and $2.2\%$ of global accuracy, since the strongest methods already reach between $85\%$ and $89\%$ on this task. This dataset also contains the single configuration in which FedPCD does not lead in tail accuracy. At $10\%$ participation and $\alpha = 0.01$, FedLF reaches a worst-$10\%$ accuracy of $65.73\%$, whereas the best FedPCD variant reaches $63.22\%$, and FedLF also attains the lower $\sigma_c$ there. FedPCD nevertheless keeps the higher global accuracy in that configuration, at $87.87\%$ against $87.00\%$.

\begin{table}[!t]
\centering
\caption{Results on CIFAR-10, CIFAR-100, and Fashion-MNIST. $\bar{a}$: global accuracy; $\sigma_{c}$: within-run std.\ across client accuracies; $\mathrm{W}_{10\%}$: worst-10\% client accuracy. Within each dataset block, bold marks the best overall value in a column and underline marks the best value from the competing method group, such that the strongest FedPCD configuration and strongest baseline are both highlighted.}
\label{tab:main_results}
\setlength{\tabcolsep}{1.8pt}
\renewcommand{\arraystretch}{1.32}
\scriptsize
\resizebox{\linewidth}{!}{%
\begin{tabular}{lcccccccccccccccccc}
\toprule
Method & \multicolumn{9}{c}{5\% participation} & \multicolumn{9}{c}{10\% participation} \\
\cmidrule(lr){2-10} \cmidrule(lr){11-19}
 & \multicolumn{3}{c}{$\alpha=0.01$} & \multicolumn{3}{c}{$\alpha=0.1$} & \multicolumn{3}{c}{$\alpha=1$} & \multicolumn{3}{c}{$\alpha=0.01$} & \multicolumn{3}{c}{$\alpha=0.1$} & \multicolumn{3}{c}{$\alpha=1$} \\
\cmidrule(lr){2-4} \cmidrule(lr){5-7} \cmidrule(lr){8-10} \cmidrule(lr){11-13} \cmidrule(lr){14-16} \cmidrule(lr){17-19}
 & $\bar{a}\!\uparrow$ & $\sigma_{c}\!\downarrow$ & $\mathrm{W}_{10\%}\!\uparrow$ & $\bar{a}\!\uparrow$ & $\sigma_{c}\!\downarrow$ & $\mathrm{W}_{10\%}\!\uparrow$ & $\bar{a}\!\uparrow$ & $\sigma_{c}\!\downarrow$ & $\mathrm{W}_{10\%}\!\uparrow$ & $\bar{a}\!\uparrow$ & $\sigma_{c}\!\downarrow$ & $\mathrm{W}_{10\%}\!\uparrow$ & $\bar{a}\!\uparrow$ & $\sigma_{c}\!\downarrow$ & $\mathrm{W}_{10\%}\!\uparrow$ & $\bar{a}\!\uparrow$ & $\sigma_{c}\!\downarrow$ & $\mathrm{W}_{10\%}\!\uparrow$ \\
\midrule
\multicolumn{19}{@{}l}{\textbf{CIFAR-10}} \\
\addlinespace[1pt]
FedAvg & 48.09 & 19.55 & 18.74 & 60.96 & 13.68 & 37.08 & 71.93 & 5.78 & 61.68 & 57.47 & 15.93 & 29.93 & 64.43 & 11.67 & 43.27 & 72.94 & 5.65 & 63.28 \\
AFL & 26.82 & 18.50 & 7.19 & 46.20 & 10.83 & 25.93 & 64.08 & 5.64 & 54.42 & 43.86 & 11.37 & 22.16 & 53.79 & 8.97 & 36.99 & 65.49 & 5.53 & 56.18 \\
q-FFL & 50.62 & 17.16 & 24.00 & 60.30 & 13.12 & 37.95 & 71.18 & 5.75 & 61.03 & 56.60 & 15.02 & 30.21 & 63.31 & 11.43 & 42.77 & 71.96 & 5.61 & 62.04 \\
Ditto & 48.25 & 19.45 & 19.19 & 61.00 & 13.71 & 36.95 & 71.84 & 5.88 & 61.27 & 57.48 & 15.90 & 29.98 & 64.42 & 11.65 & 43.16 & 73.03 & 5.63 & 63.34 \\
FedMGDA+ & 39.39 & 13.90 & 16.09 & 44.70 & 11.99 & 23.86 & 55.70 & 6.30 & 44.43 & 40.33 & 13.45 & 15.82 & 45.36 & 10.36 & 26.32 & 55.96 & 6.06 & 44.94 \\
FedFV & 61.45 & 15.46 & 33.88 & 71.46 & 10.42 & 51.05 & 76.09 & \underline{4.99} & \underline{66.91} & 69.90 & 11.01 & 50.23 & 73.98 & 8.48 & 57.20 & \underline{76.73} & \underline{4.95} & \underline{67.90} \\
FedLF & \underline{65.59} & \underline{13.00} & \underline{38.58} & 70.98 & \underline{8.67} & \underline{54.23} & 73.71 & 5.18 & 64.15 & \underline{74.62} & \textbf{8.88} & \underline{56.93} & \underline{75.27} & \underline{7.50} & \underline{61.17} & 72.44 & 5.16 & 62.87 \\
FedHEAL & 47.04 & 20.02 & 17.25 & 59.99 & 13.78 & 35.33 & 71.53 & 5.93 & 60.93 & 56.78 & 16.20 & 29.03 & 64.10 & 12.04 & 41.69 & 73.03 & 5.43 & 63.80 \\
FedPW & 46.91 & 19.72 & 18.66 & 60.68 & 13.72 & 35.81 & 72.16 & 5.80 & 61.67 & 54.56 & 16.92 & 25.35 & 64.97 & 11.67 & 43.27 & 73.42 & 5.58 & 63.51 \\
CRAFT & 65.10 & 15.51 & 31.96 & \underline{72.50} & 10.50 & 51.97 & \underline{76.76} & 5.59 & 66.50 & 73.67 & 11.92 & 48.84 & 74.82 & 8.45 & 58.74 & 75.70 & 5.11 & 66.50 \\
\cmidrule(lr){1-19}
FedPCD$_{\tau=0.005}$ & 71.53 & \textbf{10.12} & 52.85 & 74.35 & 8.30 & 58.46 & 76.26 & 4.88 & 67.67 & 74.37 & 9.24 & 56.55 & 75.13 & 7.79 & 60.34 & 77.18 & 4.90 & 68.67 \\
FedPCD$_{\tau=0.01}$ & \textbf{71.90} & 10.32 & \textbf{52.88} & 74.54 & 8.16 & 58.96 & 76.58 & 5.00 & 67.22 & \textbf{74.84} & \underline{8.89} & \textbf{58.03} & 75.34 & 7.73 & 60.41 & 77.11 & 5.02 & 68.23 \\
FedPCD$_{\tau=0.02}$ & 71.78 & 10.34 & 52.40 & \textbf{74.76} & \textbf{7.88} & \textbf{60.14} & 76.79 & 4.99 & 67.81 & 74.78 & 9.18 & 56.71 & \textbf{75.69} & \textbf{7.45} & \textbf{61.28} & 77.28 & 4.85 & 68.76 \\
FedPCD$_{\tau=0.05}$ & 67.83 & 12.18 & 44.44 & 74.04 & 8.32 & 58.17 & \textbf{77.01} & 4.87 & 68.26 & 73.04 & 10.18 & 53.00 & 75.66 & 7.64 & 61.22 & 77.24 & 4.82 & 68.50 \\
FedPCD$_{\tau=0.075}$ & 66.19 & 13.26 & 42.25 & 72.44 & 8.24 & 57.66 & 76.91 & \textbf{4.80} & \textbf{68.55} & 71.43 & 11.24 & 48.99 & 75.02 & 7.76 & 60.39 & \textbf{77.51} & \textbf{4.68} & \textbf{69.24} \\
FedPCD$_{\tau=0.1}$ & 65.89 & 13.00 & 42.72 & 71.21 & 8.62 & 55.11 & 76.83 & 4.83 & 67.88 & 71.14 & 11.47 & 47.99 & 74.57 & 7.92 & 59.28 & 77.18 & 4.74 & 68.70 \\
\midrule
\multicolumn{19}{@{}l}{\textbf{CIFAR-100}} \\
\addlinespace[1pt]
FedAvg & 9.08 & 9.61 & 0.24 & 14.37 & 6.46 & 5.15 & 19.60 & 4.15 & 12.72 & 11.22 & 9.85 & 0.82 & 16.28 & 6.66 & 6.81 & 19.93 & 4.06 & 13.32 \\
AFL & 5.64 & \textbf{6.72} & 0.00 & 10.50 & \textbf{5.15} & 3.11 & 13.40 & \textbf{3.58} & 7.44 & 7.70 & \textbf{7.12} & 0.20 & 12.28 & \textbf{5.40} & 4.38 & 14.55 & \textbf{3.62} & 8.48 \\
q-FFL & 8.69 & 9.23 & 0.20 & 13.67 & 6.22 & 4.51 & 18.79 & 4.10 & 12.08 & 10.09 & 9.31 & 0.50 & 15.01 & 6.32 & 5.82 & 19.21 & 4.11 & 12.35 \\
Ditto & 8.92 & 9.53 & 0.24 & 14.41 & 6.40 & 5.09 & 19.44 & 4.20 & 12.74 & 11.22 & 9.72 & 0.82 & 16.27 & 6.68 & 6.93 & 20.07 & 4.09 & 13.21 \\
FedMGDA+ & 8.55 & 9.68 & 0.18 & 14.58 & 6.40 & 5.46 & 22.36 & 4.30 & 15.18 & 10.83 & 9.68 & 0.70 & 16.79 & 6.59 & 7.07 & 23.74 & 4.30 & 16.88 \\
FedFV & 14.71 & 11.43 & 1.44 & 22.93 & 7.25 & \underline{12.25} & 27.52 & 4.62 & 19.37 & 18.87 & 11.45 & 3.68 & 25.73 & 7.20 & 14.35 & 28.21 & 4.55 & 20.41 \\
FedLF & 8.32 & 9.94 & 0.00 & 16.61 & 6.34 & 7.33 & 18.05 & 4.03 & 11.62 & 14.75 & 11.29 & 1.10 & 20.27 & 6.18 & 10.45 & 17.82 & 3.90 & 11.26 \\
FedHEAL & 8.79 & 9.99 & 0.24 & 14.61 & 6.50 & 5.54 & 20.30 & 4.15 & 13.53 & 11.65 & 10.14 & 0.82 & 16.71 & 6.75 & 7.14 & 20.85 & 4.07 & 14.14 \\
FedPW & 8.53 & 9.55 & 0.24 & 13.66 & 6.56 & 4.55 & 20.83 & 4.13 & 14.06 & 11.52 & 10.04 & 0.93 & 17.37 & 6.81 & 7.41 & 21.97 & 4.22 & 15.32 \\
CRAFT & \underline{15.46} & 11.32 & \underline{2.05} & \underline{23.24} & 7.38 & 11.60 & \underline{29.77} & 4.58 & \underline{21.66} & \underline{22.60} & 11.49 & \underline{6.34} & \underline{28.69} & 7.01 & \underline{17.53} & \underline{31.29} & 4.62 & \underline{22.99} \\
\cmidrule(lr){1-19}
FedPCD$_{\tau=0.005}$ & 17.74 & 11.88 & \textbf{2.83} & 26.89 & 7.52 & 15.13 & 30.27 & \underline{4.61} & 22.36 & 23.29 & 11.83 & 6.71 & 29.99 & 7.19 & 18.98 & 30.33 & 4.81 & 21.94 \\
FedPCD$_{\tau=0.01}$ & 18.05 & 11.91 & 2.80 & 27.32 & 7.67 & 15.40 & 31.06 & 4.84 & 22.28 & 23.68 & 11.80 & 6.98 & 30.73 & 7.19 & 19.53 & 31.17 & 4.78 & 22.90 \\
FedPCD$_{\tau=0.02}$ & \textbf{18.33} & 11.90 & 2.78 & 28.00 & 7.63 & 16.17 & 32.66 & 4.85 & 23.89 & \textbf{24.34} & 11.89 & \textbf{7.34} & 31.99 & 7.19 & 20.28 & 32.69 & \underline{4.77} & 24.54 \\
FedPCD$_{\tau=0.05}$ & 17.54 & 12.14 & 1.96 & \textbf{28.17} & \underline{7.42} & \textbf{16.66} & \textbf{35.40} & 4.84 & \textbf{26.89} & 23.60 & 11.76 & 6.97 & \textbf{32.54} & 7.29 & \textbf{20.74} & 35.84 & 4.81 & 27.18 \\
FedPCD$_{\tau=0.075}$ & 16.44 & 11.76 & 1.83 & 26.99 & 7.56 & 15.55 & 35.10 & 4.99 & 26.24 & 22.50 & \underline{11.55} & 6.57 & 31.24 & \underline{7.12} & 19.70 & \textbf{36.47} & 4.86 & \textbf{27.70} \\
FedPCD$_{\tau=0.1}$ & 15.25 & \underline{11.70} & 1.34 & 25.36 & 7.59 & 13.41 & 33.58 & 4.86 & 24.77 & 21.51 & \underline{11.55} & 5.94 & 29.62 & 7.25 & 17.90 & 35.30 & 4.80 & 26.95 \\
\midrule
\multicolumn{19}{@{}l}{\textbf{Fashion-MNIST}} \\
\addlinespace[1pt]
FedAvg & 77.98 & 19.59 & 30.40 & 76.75 & 14.66 & 44.73 & 81.27 & 6.25 & 69.09 & 78.63 & 19.62 & 30.49 & 79.33 & 12.62 & 51.38 & 82.82 & 5.23 & 72.74 \\
AFL & 74.92 & 17.03 & 34.49 & 76.40 & 12.90 & 47.94 & 79.87 & 6.23 & 67.37 & 76.92 & 13.70 & 44.66 & 78.39 & 10.24 & 56.78 & 80.33 & 5.80 & 68.95 \\
q-FFL & 77.80 & 19.25 & 31.23 & 76.58 & 14.63 & 44.84 & 81.23 & 5.97 & 69.46 & 77.99 & 19.79 & 29.11 & 79.02 & 12.46 & 51.80 & 82.56 & 5.26 & 72.15 \\
Ditto & 77.98 & 19.59 & 30.40 & 76.75 & 14.66 & 44.73 & 81.27 & 6.25 & 69.09 & 78.63 & 19.62 & 30.49 & 79.33 & 12.62 & 51.38 & 82.82 & 5.23 & 72.74 \\
FedMGDA+ & 72.42 & 23.07 & 18.26 & 75.59 & 14.47 & 44.51 & 79.77 & 6.14 & 67.66 & 74.55 & 21.34 & 22.22 & 76.38 & 12.80 & 49.56 & 80.72 & 5.50 & 70.13 \\
FedFV & 82.99 & 13.39 & 53.63 & 85.37 & 8.59 & 66.46 & 85.54 & 4.52 & 76.94 & 85.06 & 13.12 & 53.53 & 86.32 & 7.76 & 70.37 & 87.13 & 4.28 & 79.01 \\
FedLF & \underline{84.79} & \underline{10.50} & \underline{62.13} & \underline{86.34} & \underline{7.59} & \underline{70.40} & 86.54 & 4.32 & 78.22 & \underline{87.00} & \textbf{9.34} & \textbf{65.73} & \underline{87.27} & \textbf{6.50} & \underline{73.80} & 87.77 & \underline{4.04} & 79.84 \\
FedHEAL & 77.04 & 20.00 & 29.50 & 76.52 & 15.08 & 43.37 & 81.82 & 6.00 & 69.87 & 78.13 & 19.64 & 30.01 & 79.32 & 12.77 & 50.64 & 82.87 & 5.43 & 72.30 \\
FedPW & 78.53 & 19.50 & 33.76 & 81.60 & 11.98 & 55.55 & 87.16 & 4.21 & 78.95 & 81.97 & 15.59 & 45.58 & 82.97 & 10.53 & 61.47 & 87.70 & 4.24 & 79.64 \\
CRAFT & 81.26 & 18.01 & 36.53 & 84.07 & 10.32 & 62.35 & \underline{87.18} & \underline{4.15} & \underline{79.25} & 84.08 & 13.83 & 51.21 & 86.79 & 8.07 & 70.03 & \underline{88.00} & 4.16 & \underline{80.15} \\
\cmidrule(lr){1-19}
FedPCD$_{\tau=0.005}$ & \textbf{86.98} & \textbf{10.21} & \textbf{62.96} & \textbf{87.71} & 7.43 & \textbf{72.05} & 87.56 & 4.05 & 79.91 & 87.61 & 10.85 & 62.08 & \textbf{88.22} & \underline{6.99} & \textbf{73.88} & 88.53 & 3.92 & 80.92 \\
FedPCD$_{\tau=0.01}$ & 86.51 & 10.66 & 61.63 & 87.61 & \textbf{7.38} & 71.85 & \textbf{87.81} & \textbf{3.96} & \textbf{80.39} & \textbf{87.87} & \underline{10.43} & \underline{63.22} & 88.10 & 7.10 & 73.26 & \textbf{88.63} & \textbf{3.91} & \textbf{81.28} \\
FedPCD$_{\tau=0.02}$ & 86.29 & 10.40 & 62.23 & 87.02 & 7.53 & 71.27 & 87.54 & 4.13 & 79.69 & 87.37 & 11.06 & 60.93 & \textbf{88.22} & 7.06 & 73.32 & 88.19 & 3.95 & 80.37 \\
FedPCD$_{\tau=0.05}$ & 85.16 & 11.89 & 57.07 & 86.49 & 7.62 & 70.51 & 87.09 & 4.16 & 79.20 & 87.43 & 11.00 & 61.00 & 87.56 & 7.41 & 71.62 & 87.92 & 4.15 & 80.05 \\
FedPCD$_{\tau=0.075}$ & 84.74 & 11.50 & 58.68 & 86.48 & 8.05 & 69.24 & 86.95 & 4.46 & 78.20 & 87.53 & 10.71 & 62.21 & 87.33 & 7.39 & 71.68 & 87.83 & 4.18 & 79.81 \\
FedPCD$_{\tau=0.1}$ & 84.28 & 12.13 & 56.79 & 86.29 & 8.18 & 68.70 & 86.68 & 4.48 & 78.07 & 87.25 & 11.53 & 59.11 & 87.38 & 7.35 & 71.88 & 87.60 & 4.20 & 79.72 \\
\bottomrule
\end{tabular}
}
\end{table}

\begin{figure*}[h]
    \centering
    \includegraphics[width=0.75\linewidth]{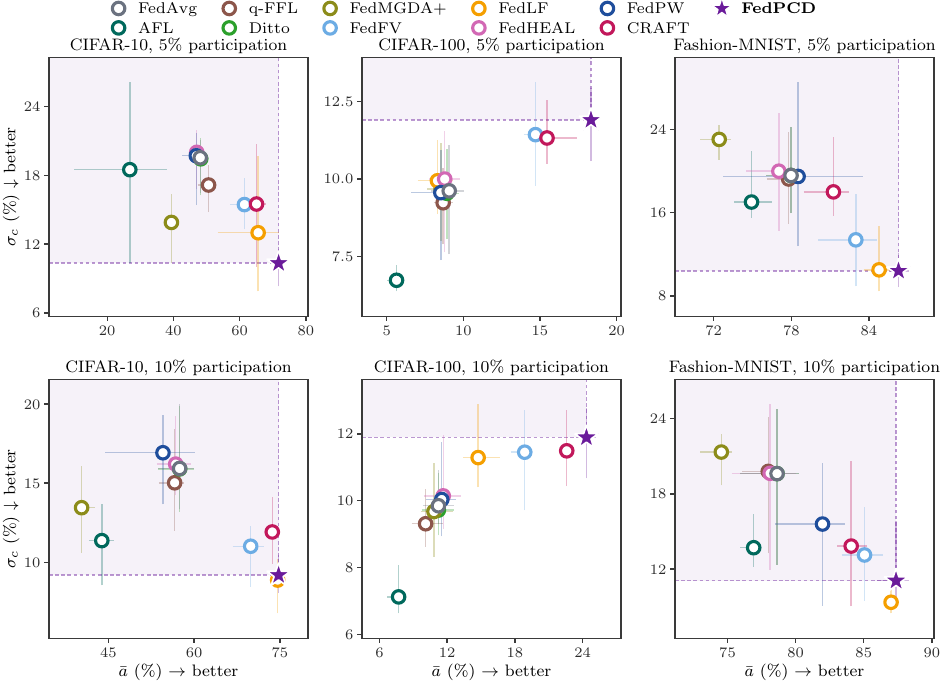}
    \caption{Global accuracy against client-accuracy spread at $\alpha=0.01$ for both participation rates. Higher accuracy and lower $\sigma_c$ are preferred. Crosses indicate variation across the 5 seeds}
    \label{fig:fairness-frontier-alpha0.01}
\end{figure*}

\subsubsection{Client-Level Fairness}
\label{sec:res-fairness}

The global accuracy alone does not show whether the improvements reach every client. The spread $\sigma_c$ cannot answer this alone either, since a method can score well on spread while holding the global accuracy down. AFL illustrates this: it has the lowest $\sigma_c$ in all 6 CIFAR-100 configurations, but its global accuracy ranges from $5.64\%$ to $14.55\%$ against a best of $18.33\%$ to $36.47\%$. A low $\sigma_c$ is therefore only meaningful next to a competitive $\bar{a}$. Fig.~\ref{fig:fairness-frontier-alpha0.01} plots the two together, so that the ideal region is the bottom-right corner of each panel. FedPCD is at least as accurate and at least as even as all 10 baselines on CIFAR-10 and Fashion-MNIST at $5\%$ participation, and at least 9 of the 10 at $10\%$.

On CIFAR-100, FedPCD is the most accurate method by a wider margin, but its $\sigma_c$ is higher than that of the baselines. This is the same effect seen from the other side, since $\sigma_c$ grows with the accuracy it is measured around. The two baselines that come closest to FedPCD in accuracy, CRAFT and FedFV, also have the two highest $\sigma_c$ among the baselines, whereas the methods whose accuracy stays near $10\%$ have a compressed spread. We complement these results with Jain's fairness index~\citep{jain1984quantitative} in App.~\ref{app:fairness}.

Fig.~\ref{fig:client-distribution-all} shows the full distribution of client accuracies. We show every method as a violin over the client accuracies pooled across 5 seeds, ordered by global accuracy. FedPCD has the highest global accuracy in every panel. On CIFAR-10 at $5\%$ participation it attains the highest $\bar{a}$, the lowest $\sigma_c$ and the highest $\mathrm{W}_{10\%}$ at once, and its worst client across the 5 runs reaches $41.0\%$, whereas the worst client of AFL reaches $0.0\%$. Fig.~\ref{fig:client-distribution-comparison} isolates FedPCD, FedAvg, and the strongest non-FedPCD baseline in each setting to compare distribution shapes. Compared to FedAvg, the entire distribution is shifted to the right, with significantly less spread. Compared to the next-best baseline, more of the mass is concentrated in the upper tail. The exception is again FedLF. At $10\%$ participation on CIFAR-10 and Fashion-MNIST, FedLF attains a lower $\sigma_c$ and a higher $\mathrm{W}_{10\%}$ than FedPCD at $\tau = 0.02$ despite a lower $\bar{a}$. The largest gap is on Fashion-MNIST, where FedLF reaches a worst-$10\%$ accuracy of $65.73\%$ against $63.22\%$ for FedPCD at its best $\tau$.

\begin{figure*}[t]
    \centering
    \includegraphics[width=1\linewidth]{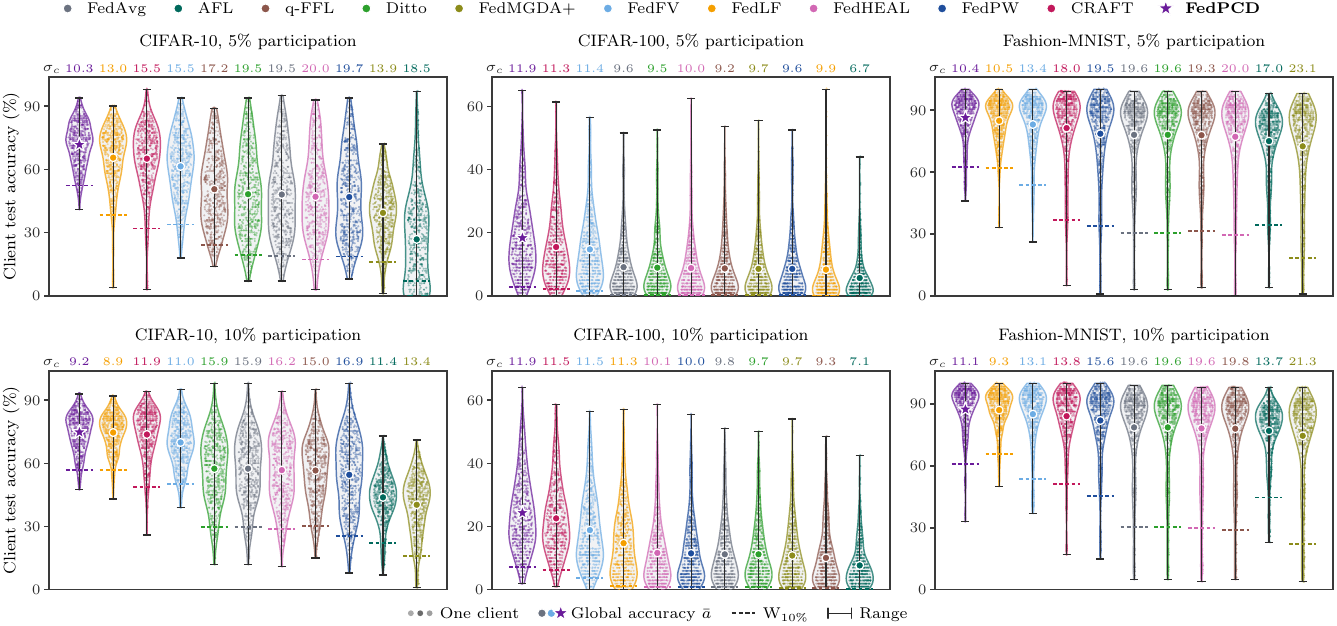}
    \caption{Client test-accuracy distributions at $\alpha=0.01$ across datasets and participation rates. Methods are ordered by global accuracy with each violin representing the kernel density of client accuracies pooled over the 5 seeds.}
    \label{fig:client-distribution-all}
\end{figure*}

\begin{figure*}[t]
    \centering
    \includegraphics[width=0.75\linewidth]{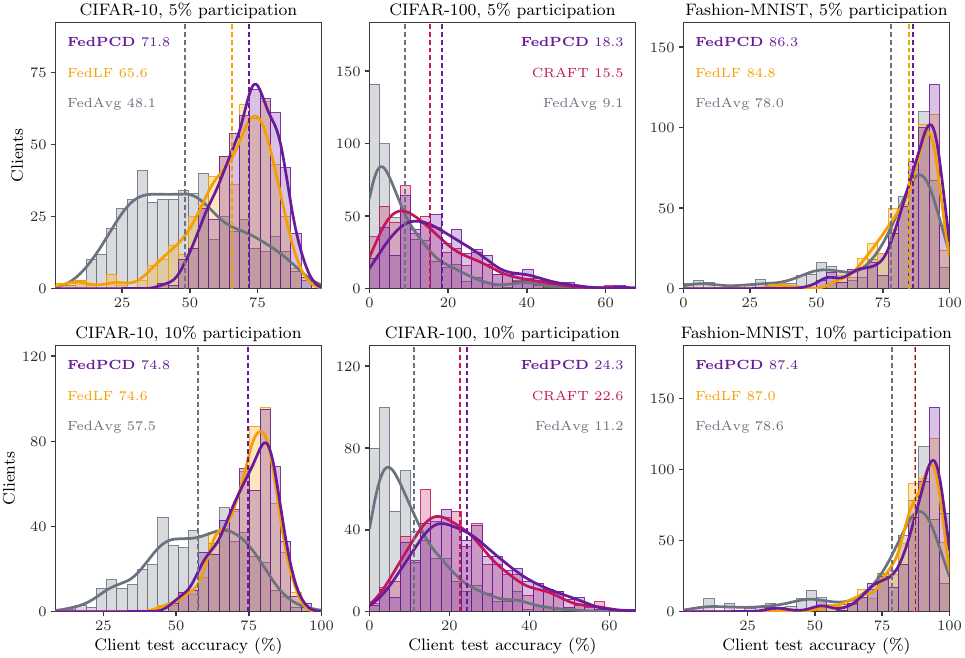}
    \caption{Client test accuracy distributions for FedPCD, FedAvg, and the strongest non-FedPCD baseline at $\alpha = 0.01$. The distribution means are shown as dashed lines and their values are reported per panel.}
    \label{fig:client-distribution-comparison}
\end{figure*}

\subsubsection{Convergence}
\label{sec:res-convergence}

\begin{figure*}[!t]
    \centering
    \includegraphics[width=0.8\linewidth]{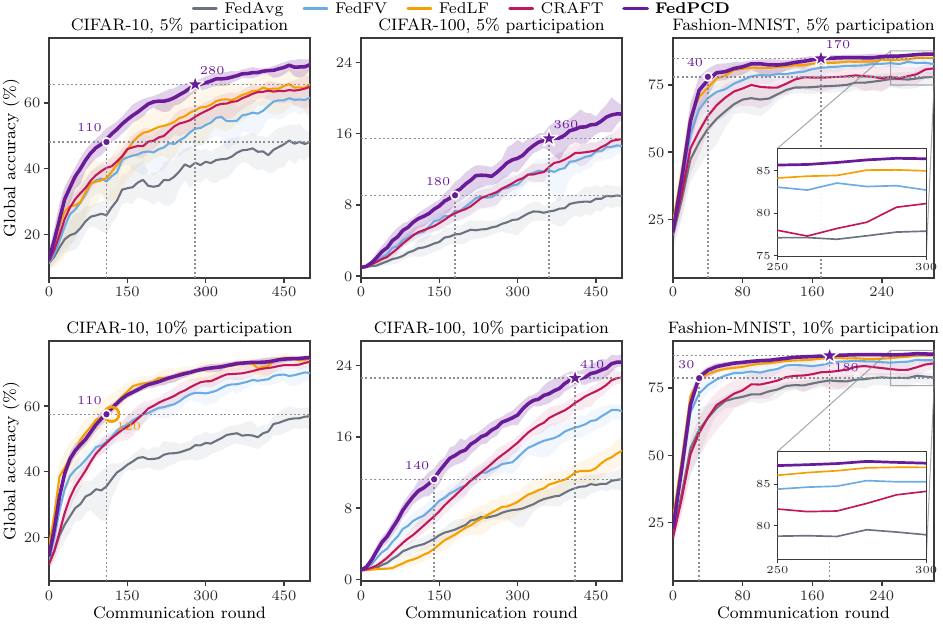}
    \caption{Global accuracy against communication rounds at $\alpha=0.01$. Dots mark when FedPCD reaches FedAvg's final accuracy, while stars mark when it reaches the strongest baseline's final accuracy. When the latter is not reached, a hollow ring marks when that baseline reaches FedAvg's final accuracy.}
    \label{fig:communication-rounds-convergence-alpha0.01}
\end{figure*}

Fig.~\ref{fig:communication-rounds-convergence-alpha0.01} compares the global accuracy over communication rounds against FedAvg and the three strongest baselines under the strongest heterogeneity setting, $\alpha = 0.01$. The recurring pattern is that FedPCD reaches the final
performance levels of both FedAvg and the strongest competing methods significantly earlier in most settings, demonstrating gains in communication efficiency.

\paragraph{CIFAR-10.}
At both participation rates, FedPCD reaches FedAvg's final accuracy by round 110 out of 500, requiring only $22\%$ of the communication budget and corresponding to approximately $4.5\times$ fewer rounds to reach the same accuracy. At $5\%$ participation, FedPCD reaches the strongest baseline's final accuracy by round 280, corresponding to a further $1.8\times$ reduction in communication rounds. The advantage is smaller at $10\%$ participation, where FedLF closely tracks FedPCD and their final accuracies are matched at the end of training.

\paragraph{CIFAR-100.}
The communication efficiency advantage remains relevant in CIFAR-100. FedPCD reaches FedAvg's final accuracy by rounds 180 and 140 at $5\%$ and $10\%$ participation, respectively. This corresponds to approximately $2.8\times$ and $3.6\times$ fewer communication rounds to match FedAvg's performance at 500 rounds. The final accuracy of the strongest baseline is also reached earlier, at rounds 360 and 410, corresponding to 140 and 90 rounds before the end of training and approximately $1.4\times$ and $1.2\times$ better communication efficiency, respectively.

\paragraph{Fashion-MNIST.}
The largest relative reduction in communication rounds occurs on Fashion-MNIST. FedPCD reaches
FedAvg's final accuracy after only 40 of 300 rounds at $5\%$ participation and 30 rounds at $10\%$, corresponding to $7.5\times$ and $10\times$ reductions in communication rounds, respectively. It also reaches the strongest baseline's final accuracy by rounds 170 and 180, respectively, amounting to $1.8\times$ and $1.7\times$ savings in communication budget, respectively.

Overall, FedPCD demonstrates strong communication efficiency across all tested settings (see App.~\ref{app:all-alpha}). It reaches FedAvg's final accuracy using only $10\%$--$44\%$ of the total communication rounds, and achieves the strongest baseline's final accuracy earlier in 16 of the 18 configurations. These results show that FedPCD not only achieves stronger final performance, but also reaches a competitive performance in substantially fewer rounds.

\section{Conclusion}
This work presents FedPCD, a priority-constrained federated optimization framework that separates global federated aggregation (FedAvg) from the protection of individual client objectives. Instead of replacing the global update with a re-weighted or symmetric multi-objective direction, FedPCD retains FedAvg as the primary descent direction and applies the minimal correction required to satisfy client-level constraints. To address partial participation, we incorporate a client history mechanism to preserve recent client directions as secondary constraints. Under deterministic full participation, we establish an $\mathcal{O}(1/T)$ convergence rate, further showing client-level progress toward simultaneous stationarity under persistent feasibility.

Extensive experiments on CIFAR-10, CIFAR-100 and FashionMNIST under 2 participation levels and 3 data heterogeneity levels show superior performance in comparison to 10 tested fairness-aware FL baselines, maintaining the highest global accuracy in all cases, better worst-$10\%$ client accuracy in 17 out of 18 tested configurations, while converging in significantly fewer communication rounds.

Several directions remain open as limitations of the FedPCD framework. Firstly, we use a fixed $\tau$ parameter across all secondary constraints, although the observed sensitivity to heterogeneity suggests that an adaptive client- or round-dependent parameter may further improve our gains. Additionally, the current history mechanism requires the server to observe each client's pseudo-gradient individually rather than only their aggregate, limiting applicability to secure aggregation systems. This motivates privacy-compatible variants of FedPCD. Finally, we note that the evaluation is restricted to $N = 100$ clients and one local epoch, so the behavior of FedPCD under many local steps, where pseudo-gradient directions degrade through client drift, is not established.

\subsubsection*{Acknowledgments}
This research was supported in part by the High-Performance Computing (HPC) environment at the American University of Sharjah. The authors acknowledge the AUS-HPC team for providing the computational infrastructure and technical support.

\bibliography{FedPCD}
\bibliographystyle{tmlr}

\appendix
\section{Reproducibility Details}
\label{app:reproducibility}

\paragraph{Model and Training Details.}
For CIFAR-10, we use a three-layer CNN with $3\times3$ convolutional layers containing 32, 64, and 128 channels. Each convolutional layer is followed by GroupNorm, ReLU activation, and $2\times2$ max pooling, followed by a classification head with dimensions $2048\rightarrow256\rightarrow10$. For CIFAR-100, we use a ResNet-18 with a $3\times3$ stride-one stem and no initial max pooling. Each of the four residual stages contains two residual blocks with 64, 128, 256, and 512 channels, respectively, with GroupNorm used throughout the network. For FashionMNIST, images are flattened before being passed to an MLP with dimensions $784\rightarrow256\rightarrow128\rightarrow10$ and ReLU activations after each hidden layer. The local batch sizes are 128 and 64 for CIFAR-10 and CIFAR-100, respectively, while full-batch local training is used for FashionMNIST. The corresponding local and global learning rates are 0.05, 0.01, and 0.1.

\paragraph{Experimental Control.}
All experiments use five paired random seeds. For a given dataset, heterogeneity setting, and seed, all compared methods use the same client data partition, initial model parameters, sequence of participating clients, minibatch ordering, and client test partitions. Clients are sampled uniformly without replacement at each communication round. Local optimization uses SGD with cross-entropy loss, without momentum or weight decay. All performance metrics are reported using the model obtained at the final communication round.

\paragraph{Baselines Implementation and Hyperparameters.}
The baselines included in our experimental evaluation were selected based on their public implementation availability to support faithful comparisons. For methods with tunable method-specific hyperparameters, we evaluate the grid in Table~\ref{tab:hyperparameters} and report the best performing hyperparameters per dataset after averaging the five seeds. For a fair comparison, the methods that implement similar history-window mechanisms, FedFV \cite{wangFederatedLearningFair2021} and FedLF \cite{panFedLFLayerWiseFair2024}, were both evaluated at $H=10$.

\begin{table}[!h]
\centering
\caption{Hyperparameter search spaces.}
\label{tab:hyperparameters}
\begin{tabular}{ll}
\toprule
Method & Parameters \\
\midrule
AFL & $\eta_{\lambda} \in \{0.01,\,0.1,\,0.5\}$ \\
q-FFL & $q \in \{0.1,\,1,\,10\}$ \\
Ditto & $\lambda \in \{0.1,\,1,\,2\}$ \\
FedMGDA+ & $\epsilon \in \{0.1,\,0.5,\,1\}$ \\
FedFV & $H=10, \quad \alpha \in \{0.1,\,0.3,\, 0.5\}$ \\
FedLF & $H=10$\\
FedHEAL & $\beta = 0.4,\quad \tau_{\mathrm{HEAL}} \in \{0.2,\,0.3,\,0.4\}$ \\
FedPW & $\beta = 0.5,\quad c \in \{0.1,\,0.3,\,0.5\}$ \\
FedPCD & $\tau \in \{0.005,\,0.01,\,0.02,\,0.05,\,0.075,\,0.1\},\quad
H \in \{0,\,1,\,\ldots,\,10\}$ \\
\bottomrule
\end{tabular}
\end{table}

\section{Expanded Theoretical Analysis}
\label{app:theory}

\subsection{Assumptions and Proof Overview}
\label{app:theory_assumptions}

The proof controls the individual client gradients before bounding the global gradient. This is necessary because the global gradient can be small through cancellation among client gradients, whereas simultaneous client stationarity requires every client gradient to be small. We first show that the alignment condition prevents such cancellation among the pseudo-gradients. We then bound the error from local training and establish a decrease in the global objective proportional to the squared average of the client gradient norms. Summing this decrease proves both Theorem~\ref{thm:fedpcd_stationarity} and Corollary~\ref{cor:fedpcd_cms}.

\paragraph{Setting and regularity.}
We analyze the updates in Equations~\ref{eq:fedpcd_client_direction}--\ref{eq:fedpcd_global_update} under full participation, deterministic local gradient descent, a common fixed number $K\geq1$ of local steps, and $H=0$. The weights $p_i>0$ are fixed and sum to one. Each $F_i$ has an $L$-Lipschitz gradient, with $L>0$, and $F$ has a finite lower bound $F_\star$. Thus, $F$ is also $L$-smooth. The local and server learning rates are positive constants satisfying the conditions of Theorem~\ref{thm:fedpcd_stationarity}; they do not depend on the number of rounds $T$. We require neither convexity nor a uniform bound on the local gradients. Stochastic gradients, partial participation, and historical directions are outside this theorem's scope.

\paragraph{Feasibility and uniform alignment.}
Recall that $\mathcal I_t=\{i:g_{i,t}\neq0\}$ denotes the clients with nonzero pseudo-gradients. At every round with $\mathcal I_t\neq\varnothing$, we assume that the QP is feasible and its solution satisfies Equation~\ref{eq:main_alignment_margin} for every $i\in\mathcal I_t$, with the same $\gamma\in(0,1]$ throughout the run. This condition bounds the cosine of the angle between the QP solution and each nonzero client pseudo-gradient away from zero. The constant $\gamma$ describes the resulting geometry and is not a user-selected parameter.

Feasibility and uniform alignment play different roles. Since $g_{i,t}$ and $\tilde g_{i,t}$ have the same direction, the QP constraint implies
\[
    \frac{g_{i,t}^{\top}\tilde d_t^\star}
         {\lVert g_{i,t}\rVert\lVert\tilde d_t^\star\rVert}
    \geq
    \frac{\tau\lVert\tilde g_{i,t}\rVert}
         {\lVert\tilde d_t^\star\rVert}>0,
    \qquad i\in\mathcal I_t.
\]
Thus, feasibility with $\tau>0$ gives positive alignment at the current round, but does not provide a positive lower bound valid for all rounds. Equation~\ref{eq:main_alignment_margin} supplies this additional bound. It is a substantive restriction on the client directions, not a consequence of smoothness. In particular, two nonzero directions that point exactly opposite to one another cannot both satisfy positive progress constraints. The theorem, therefore, does not apply to arbitrary heterogeneous objectives without this geometric condition.

\subsection{Supporting Results}
\label{app:theory_supporting}

\begin{lemma}[Client alignment and aggregate magnitude]
\label{lem:theory_geometry}
Write
\[
    G_t=\sum_{i=1}^{N}p_i\lVert g_{i,t}\rVert.
\]
Here $G_t$ is the weighted average of the client pseudo-gradient magnitudes, whereas $\lVert\bar g_t\rVert$ is the magnitude of their weighted aggregate. Under the alignment condition,
\begin{align}
    g_{i,t}^{\top}d_t
    &\geq \gamma\lVert g_{i,t}\rVert\lVert\bar g_t\rVert,
    \qquad \lVert d_t\rVert=\lVert\bar g_t\rVert,
    \label{eq:theory_client_progress}\\
    \gamma G_t&\leq\lVert\bar g_t\rVert\leq G_t.
    \label{eq:theory_aggregate_norm}
\end{align}
In particular, aggregation cannot make the update arbitrarily small relative to the average client pseudo-gradient magnitude while $\gamma$ remains fixed.
\end{lemma}

\begin{proof}
First suppose $\mathcal I_t\neq\varnothing$. A positive right-hand side in at least one QP constraint excludes $\tilde d_t^\star=0$. Define the unit direction $u_t=\tilde d_t^\star/\lVert\tilde d_t^\star\rVert$. Equation~\ref{eq:main_alignment_margin} gives $g_{i,t}^{\top}u_t\geq \gamma\lVert g_{i,t}\rVert$. Multiplying by $p_i$ and summing yields
\[
    \gamma G_t
    \leq\sum_i p_i g_{i,t}^{\top}u_t
    =\bar g_t^{\top}u_t
    \leq\lVert\bar g_t\rVert,
\]
where the last inequality uses Cauchy--Schwarz and $\lVert u_t\rVert=1$. The triangle inequality gives the upper bound $\lVert\bar g_t\rVert\leq G_t$. Since $p_i>0$ and at least one pseudo-gradient is nonzero, the lower bound implies $\lVert\bar g_t\rVert>0$.

The rescaling in Equation~\ref{eq:fedpcd_rescaling} gives $d_t=\lVert\bar g_t\rVert u_t$. Multiplying the client alignment inequality by $\lVert\bar g_t\rVert$ proves Equation~\ref{eq:theory_client_progress}, including the norm identity. If all pseudo-gradients are zero, then $G_t=\lVert\bar g_t\rVert=0$ and $d_t=0$, so both conclusions hold directly.
\end{proof}

\begin{lemma}[Relative local drift]
\label{lem:theory_drift}
Let $r$ be the dimensionless quantity defined in Equation~\ref{eq:main_drift_factors}. Each pseudo-gradient satisfies
\begin{equation}
    g_{i,t}=K\nabla F_i(\theta_t)+e_{i,t},
    \qquad
    \lVert e_{i,t}\rVert
    \leq Kr\lVert\nabla F_i(\theta_t)\rVert.
    \label{eq:theory_relative_drift}
\end{equation}
If $r<1$, then
\begin{equation}
    K(1-r)\lVert\nabla F_i(\theta_t)\rVert
    \leq\lVert g_{i,t}\rVert
    \leq K(1+r)\lVert\nabla F_i(\theta_t)\rVert.
    \label{eq:theory_relative_norms}
\end{equation}
The drift bound is relative to the client's gradient at the round-start model. Consequently, this error vanishes as that gradient vanishes, even with a fixed local learning rate.
\end{lemma}

\begin{proof}
Fix a client and round, and abbreviate $a=\lVert\nabla F_i(\theta_t)\rVert$ and $z_k=\lVert\theta_{i,t}^{(k)}-\theta_t\rVert$. Smoothness gives $\lVert\nabla F_i(\theta_{i,t}^{(k)})\rVert\leq a+Lz_k$. Applying the triangle inequality to the local update therefore yields
\[
    z_{k+1}
    \leq z_k+\eta_\ell(a+Lz_k)
    =(1+L\eta_\ell)z_k+\eta_\ell a.
\]
Starting from $z_0=0$ and unrolling this recurrence,
\begin{equation}
    z_k\leq\eta_\ell a\sum_{s=0}^{k-1}(1+L\eta_\ell)^s
    =\frac{(1+L\eta_\ell)^k-1}{L}\,a.
    \label{eq:theory_local_displacement}
\end{equation}
The formula includes $k=0$, with the empty sum interpreted as zero.

In the deterministic setting, the pseudo-gradient is the sum of the $K$ local gradients. Subtracting $K\nabla F_i(\theta_t)$ and using smoothness once more gives
\begin{align*}
    \lVert e_{i,t}\rVert
    &=\left\lVert\sum_{k=0}^{K-1}
      \bigl[\nabla F_i(\theta_{i,t}^{(k)})-\nabla F_i(\theta_t)\bigr]
      \right\rVert\\
    &\leq L\sum_{k=0}^{K-1}z_k\\
    &\leq a\sum_{k=0}^{K-1}\bigl[(1+L\eta_\ell)^k-1\bigr]
     =Kra.
\end{align*}
Applying the triangle and reverse triangle inequalities to $g_{i,t}=K\nabla F_i(\theta_t)+e_{i,t}$ proves Equation~\ref{eq:theory_relative_norms}. These inequalities also cover $a=0$, in which case the deterministic local trajectory remains at $\theta_t$. When $r<1$, Equation~\ref{eq:theory_relative_norms} further shows that $g_{i,t}=0$ if and only if $\nabla F_i(\theta_t)=0$.
\end{proof}

\subsection{Proof of the Finite-Time Convergence Rate}
\label{app:theory_convergence}

\begin{proof}[Proof of Theorem~\ref{thm:fedpcd_stationarity}]
Define the weighted average of the client gradient norms by
\begin{equation}
    S_t=\sum_{i=1}^{N}p_i\lVert\nabla F_i(\theta_t)\rVert.
    \label{eq:theory_St}
\end{equation}
The proof bounds $S_t^2$, which controls both global and client stationarity. Recall from Equation~\ref{eq:main_drift_factors} that
$\gamma_{\mathrm{eff}}=\gamma(1-r)-r$. The assumption $\gamma_{\mathrm{eff}}>0$ is equivalent to $r<\gamma/(1+\gamma)$, and therefore implies $r<1$.

\paragraph{From pseudo-gradient alignment to true-gradient alignment.}
The local error can reduce the client alignment by at most $\lVert e_{i,t}\rVert\lVert d_t\rVert$. Using both lemmas gives
\begin{align}
    \nabla F_i(\theta_t)^{\top}d_t
    &=\frac{1}{K}(g_{i,t}-e_{i,t})^{\top}d_t\nonumber\\
    &\geq\frac{\gamma}{K}\lVert g_{i,t}\rVert\lVert\bar g_t\rVert
       -r\lVert\nabla F_i(\theta_t)\rVert\lVert\bar g_t\rVert\nonumber\\
    &\geq\bigl[\gamma(1-r)-r\bigr]
       \lVert\nabla F_i(\theta_t)\rVert\lVert\bar g_t\rVert\nonumber\\
    &=\gamma_{\mathrm{eff}}\lVert\nabla F_i(\theta_t)\rVert\lVert\bar g_t\rVert.
    \label{eq:theory_true_progress}
\end{align}
Thus, $\gamma_{\mathrm{eff}}$ is the alignment left after accounting for local drift. Because the model moves along $-d_t$, positive alignment means a negative directional derivative of the client loss.

Weighting Equation~\ref{eq:theory_relative_norms} by $p_i$ and summing gives $K(1-r)S_t\leq G_t\leq K(1+r)S_t$. Combining this with Equation~\ref{eq:theory_aggregate_norm} yields
\begin{equation}
    \gamma K(1-r)S_t\leq\lVert\bar g_t\rVert\leq K(1+r)S_t.
    \label{eq:theory_aggregate_norm_bounds}
\end{equation}
Similarly, weighting Equation~\ref{eq:theory_true_progress} and summing gives
\begin{equation}
    \nabla F(\theta_t)^{\top}d_t\geq \gamma_{\mathrm{eff}}S_t\lVert\bar g_t\rVert.
    \label{eq:theory_global_progress}
\end{equation}

\paragraph{A decrease proportional to $S_t^2$.}
Smoothness of $F$ and $\lVert d_t\rVert=\lVert\bar g_t\rVert$ imply
\begin{align}
    F(\theta_{t+1})
    &\leq F(\theta_t)-\eta_g\gamma_{\mathrm{eff}}S_t\lVert\bar g_t\rVert
          +\frac{L\eta_g^2}{2}\lVert\bar g_t\rVert^2\nonumber\\
    &\leq F(\theta_t)-\eta_g
       \left[\gamma_{\mathrm{eff}}-\frac{L\eta_gK(1+r)}{2}\right]S_t\lVert\bar g_t\rVert.
    \label{eq:theory_descent_pre}
\end{align}
The second line uses $\lVert\bar g_t\rVert^2\leq K(1+r)S_t\lVert\bar g_t\rVert$, which follows from the upper bound in Equation~\ref{eq:theory_aggregate_norm_bounds}. The server learning-rate condition in Equation~\ref{eq:fedpcd_server_stepsize} makes the bracket at least $\gamma_{\mathrm{eff}}/2$. Applying the lower bound in Equation~\ref{eq:theory_aggregate_norm_bounds} then gives
\begin{equation}
    F(\theta_{t+1})\leq F(\theta_t)
    -\frac{\eta_g\gamma_{\mathrm{eff}}\gamma K(1-r)}{2}S_t^2.
    \label{eq:theory_descent}
\end{equation}
The rescaling in Equation~\ref{eq:fedpcd_rescaling} makes the smoothness remainder shrink with the client gradient magnitudes, allowing it to be absorbed into the descent term.

\paragraph{Summation over communication rounds.}
Summing Equation~\ref{eq:theory_descent} cancels the intermediate objective values. Since $F(\theta_T)\geq F_\star$,
\begin{equation}
    \sum_{t=0}^{T-1}S_t^2
    \leq\frac{2[F(\theta_0)-F_\star]}
    {\eta_g\gamma_{\mathrm{eff}}\gamma K(1-r)}.
    \label{eq:theory_summability}
\end{equation}
Finally, the triangle inequality gives $\lVert\nabla F(\theta_t)\rVert =\lVert\sum_i p_i\nabla F_i(\theta_t)\rVert\leq S_t$. Divide Equation~\ref{eq:theory_summability} by $T$ to obtain Equation~\ref{eq:fedpcd_convergence}. Its coefficient is independent of $T$ when the learning rates and the uniform alignment margin are fixed.
\end{proof}

\paragraph{Constant local learning rates.}
For $K=1$, Equation~\ref{eq:main_drift_factors} gives $r=0$ and $\gamma_{\mathrm{eff}}=\gamma$. For $K>1$ and a given positive uniform margin $\gamma$, a sufficient numerical condition for $\gamma_{\mathrm{eff}}\geq\gamma/2$ is
\begin{equation}
    0<\eta_\ell\leq
    \frac{\log\!\left(1+\frac{\gamma}{2(1+\gamma)}\right)}
    {L(K-1)}.
    \label{eq:theory_local_stepsize}
\end{equation}
Indeed, each term in the average defining $r$ is at most its last term, so
\[
    r\leq(1+L\eta_\ell)^{K-1}-1
      \leq e^{L\eta_\ell(K-1)}-1
      \leq\frac{\gamma}{2(1+\gamma)}.
\]
Substituting into $\gamma_{\mathrm{eff}}=\gamma-(1+\gamma)r$ proves the claim. Thus, the drift restriction does not require a local learning rate that decreases with $T$. This choice controls local drift; it does not itself establish feasibility or the uniform margin, which remain assumptions on the resulting trajectory.

\paragraph{Scope of client protection.}
Equation~\ref{eq:theory_true_progress} guarantees a descent direction for each nonstationary client. For a finite server step, smoothness gives
\begin{equation}
    F_i(\theta_{t+1})-F_i(\theta_t)
    \leq-\eta_g\gamma_{\mathrm{eff}}\lVert\nabla F_i(\theta_t)\rVert\lVert\bar g_t\rVert
        +\frac{L\eta_g^2}{2}\lVert\bar g_t\rVert^2.
    \label{eq:theory_client_loss}
\end{equation}
For a nonstationary client, $\lVert\bar g_t\rVert>0$ and this upper bound is negative whenever $0<\eta_g<2\gamma_{\mathrm{eff}}\lVert\nabla F_i(\theta_t)\rVert/(L\lVert\bar g_t\rVert)$. The theorem's fixed server step guarantees global objective descent; it need not meet this additional condition for every client at every round. Client protection here therefore means directional progress, not monotonic reduction of every client loss or a guarantee about client accuracy.

\subsection{Client Stationarity and Composite Multi-Stationarity}
\label{app:theory_cms}

\begin{proof}[Proof of Corollary~\ref{cor:fedpcd_cms}]
Let $p_{\min}=\min_i p_i>0$. When the square of $S_t$ is expanded, its cross terms are nonnegative. Hence,
\begin{align*}
    S_t^2
    &\geq\sum_i p_i^2\lVert\nabla F_i(\theta_t)\rVert^2\\
    &\geq p_{\min}\sum_i p_i\lVert\nabla F_i(\theta_t)\rVert^2,
\end{align*}
where the second inequality uses $p_i^2\geq p_{\min}p_i$. Combining this with Equation~\ref{eq:theory_summability} gives
\begin{equation}
    \frac{1}{T}\sum_{t=0}^{T-1}\sum_i p_i\lVert\nabla F_i(\theta_t)\rVert^2
    \leq\frac{2[F(\theta_0)-F_\star]}
    {p_{\min}\eta_g\gamma_{\mathrm{eff}}\gamma K(1-r)T}.
    \label{eq:theory_client_stationarity_bound}
\end{equation}
This proves the finite-time statement for fixed client weights.

For an infinite run with the same constants, Equation~\ref{eq:theory_summability} bounds every partial sum of the nonnegative series $\sum_t S_t^2$. The series is therefore finite, implying $S_t\to0$. This uses summability, not merely a vanishing average. Since
$p_i\lVert\nabla F_i(\theta_t)\rVert\leq S_t$ and $p_i>0$,
\begin{equation}
    \lVert\nabla F_i(\theta_t)\rVert\longrightarrow0,
    \qquad i=1,\ldots,N.
    \label{eq:theory_client_vanishing}
\end{equation}
If a subsequence $\theta_{t_j}$ converges to $\theta^\star$, continuity gives
$\nabla F_i(\theta^\star)=\lim_j\nabla F_i(\theta_{t_j})=0$ for every client.
Since $\nabla F=\sum_i p_i\nabla F_i$, the primary global objective is stationary there as well. This is CMS in the smooth setting considered here \citep{varamNotAllObjectives2026}.
\end{proof}

\paragraph{What the limit statement guarantees.}
Bounded iterates ensure that at least one accumulation point exists. Without boundedness, the result still proves that every client gradient vanishes, but it does not assert the existence of a finite limiting parameter. Even with boundedness, the proof does not require convergence of the entire parameter sequence to a single point.

\paragraph{Why CMS is stronger than Pareto stationarity.}
For the client objectives, Pareto stationarity means that there exist weights $\alpha_i\geq0$, with $\sum_i\alpha_i=1$, such that
$\sum_i\alpha_i\nabla F_i(\theta)=0$.
CMS requires each gradient in this sum to be zero individually. Consequently,
\begin{equation}
    \underbrace{\{\theta:\nabla F_i(\theta)=0\ \forall i\}}_{\text{CMS}}
    \subseteq
    \underbrace{\{\theta:\nabla F(\theta)=0\}}_{\text{global stationarity}}
    \subseteq
    \underbrace{\{\theta:0\in\operatorname{conv}\{\nabla F_i(\theta)\}_{i=1}^{N}\}}_{\text{Pareto stationarity}}.
    \label{eq:theory_stationarity_sets}
\end{equation}
The first inclusion follows from the definition of $F$; the second follows by choosing $\alpha_i=p_i$. Nonzero gradients can cancel, so the reverse implications need not hold. The uniform alignment condition prevents this cancellation from invalidating the client-gradient bound along the analyzed trajectory. Accordingly, the CMS conclusion is conditional on that geometry; it does not establish that every collection of heterogeneous client objectives admits a common stationary point.

\subsection{Comparative Analysis with CRAFT}
\label{app:craft_comparison}
In this section, we analytically compare FedPCD to CRAFT~\citep{wangCRAFTConflictResolvedAggregation2026}, which we identify as the closest methodological neighbor in terms of the implementation and algorithmic steps. Both methods correct a reference direction according to per-client alignment requirements. CRAFT prescribes exact client-alignment values, whereas FedPCD imposes lower bounds on client progress. For completeness, we provide a visual demonstration of one gradient step involving 3 clients in Fig.~\ref{fig:comparisons}. We also include FedMGDA+, FedFV, and FedLF as all these methods apply a similar modification to the aggregation step. Each panel resolves the same three client updates, with $\tilde g_{3,t}$ (red) in conflict with the other two. FedMGDA+, FedFV, and FedLF all reshape the aggregate itself through min-norm search, pairwise projection, or the addition of fairness gradient to the hull, respectively. CRAFT corrects its reference direction to satisfy prescribed per-client rates $\rho_t$. FedPCD, on the other hand, maintains the FedAvg-like direction $\tilde g_{0,t}$ as the primary objective and imposes per-client constraints $\tilde g_{i,t}^{\top}\tilde d \ge \tau\lVert\tilde g_{i,t}\rVert^2$. In the illustrated case, FedPCD projects $\tilde g_{0,t}$ onto the feasible polyhedron $\mathcal{F}_{\tau}$ (shaded). Since only the third client constraint is active, the solution takes the form $\tilde d_t^\star=\tilde g_{0,t}+\mu_3^\star\tilde g_{3,t}$.

\begin{figure*}[h]
    \centering
    \includegraphics[width=\linewidth]{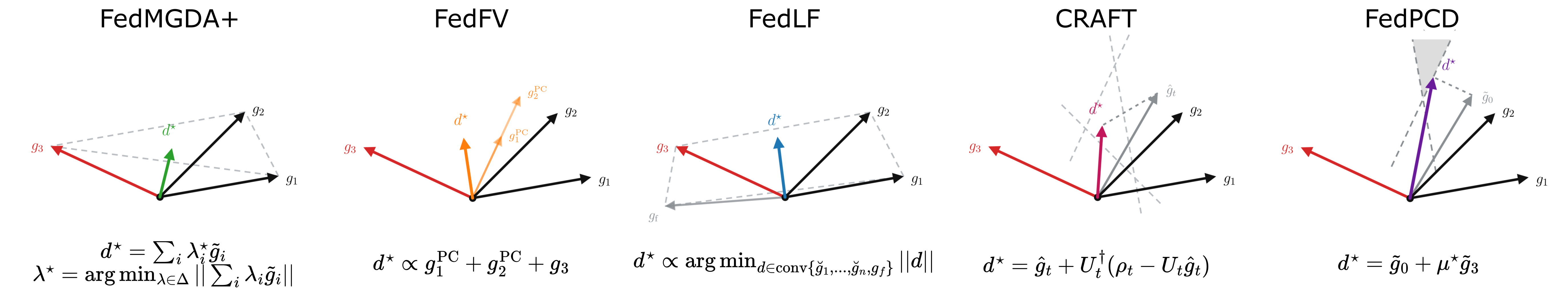}
    \caption{Comparison of FedPCD to its closest fair FL neighbors used as baselines in this study.}
    \label{fig:comparisons}
\end{figure*}

\paragraph{Inequality versus exact alignment.}
To isolate the difference between the constraint formulations from the choice of reference direction, hold the reference $v\in\mathbb R^d$, matrix $U\in\mathbb R^{m\times d}$, and targets $b\in(0,\infty)^m$ fixed. Define

$$
C_{\geq}=\{\tilde d:U\tilde d\geq b\}, \qquad C_{=}=\{\tilde d:U\tilde d=b\}.
$$

Let $\tilde d_{\geq}$ and $\tilde d_{=}$ be the respective Euclidean projections of $v$ onto these sets, whenever they are nonempty. The comparison below concerns these two projection problems with identical inputs. It is a matched-reference and matched-target comparison of the constraint geometries.

\begin{proposition}[Feasibility and minimum correction]
\label{prop:theory_formulation}
If $C_{=}$ is nonempty, so is $C_{\geq}$, and
\begin{equation}
\lVert\tilde d_{\geq}-v\rVert^2+\lVert\tilde d_{=}-\tilde d_{\geq}\rVert^2\leq\lVert\tilde d_{=}-v\rVert^2.
\label{eq:theory_projection_comparison}
\end{equation}
Thus, the inequality projection needs no larger correction to the reference. Its correction is strictly smaller if the two projected points differ.
\end{proposition}

\begin{proof}
Any vector satisfying $U\tilde d=b$ also satisfies $U\tilde d\geq b$, so $C_{=}\subseteq C_{\geq}$ and $\tilde d_{=}\in C_{\geq}$. Convexity ensures

$$
\tilde d_{\geq}+s(\tilde d_{=}-\tilde d_{\geq})\in C_{\geq}, \qquad 0\leq s\leq1.
$$

Since $\tilde d_{\geq}$ minimizes the squared distance to $v$ over $C_{\geq}$, the one-sided derivative at $s=0$ of $\lVert\tilde d_{\geq}+s(\tilde d_{=}-\tilde d_{\geq})-v\rVert^2$ is nonnegative. Therefore,

$$
(\tilde d_{\geq}-v)^\top(\tilde d_{=}-\tilde d_{\geq})\geq0.
$$

Expanding the squared distance from $v$ to $\tilde d_{=}$ gives
\begin{align*}
\lVert\tilde d_{=}-v\rVert^2
=\lVert\tilde d_{\geq}-v\rVert^2+\lVert\tilde d_{=}-\tilde d_{\geq}\rVert^2+2(\tilde d_{\geq}-v)^\top(\tilde d_{=}-\tilde d_{\geq})\
\geq\lVert\tilde d_{\geq}-v\rVert^2+\lVert\tilde d_{=}-\tilde d_{\geq}\rVert^2.
\end{align*}
This proves Equation~\ref{eq:theory_projection_comparison}. If $\tilde d_{=}\neq\tilde d_{\geq}$, then $\lVert\tilde d_{=}-\tilde d_{\geq}\rVert^2>0$, and hence $\lVert\tilde d_{\geq}-v\rVert<\lVert\tilde d_{=}-v\rVert$.
\end{proof}

Feasibility of the inequalities does not conversely guarantee feasibility of the equalities. For example, take rows $(1,0)$, $(0,1)$, and $(1,1)/\sqrt{2}$ and set $b=\tau\mathbf{1}$ with $\tau>0$. The vector $(\tau,\tau)$ satisfies all inequalities. The first two equalities force this vector, but the third then requires $\sqrt{2}\tau=\tau$, which is impossible. Likewise, the inequality projection leaves $v$ unchanged whenever $Uv\geq b$, including when some targets are exceeded; the equality projection does so only when $Uv=b$. This result concerns only the constrained directions before the rescaling in Equation~\ref{eq:fedpcd_rescaling}.

\paragraph{Published convergence bounds.}
For one local step, CRAFT's Theorem A.4 bounds the average squared global gradient by $A/(\eta_gT)+B\eta_g$, with constants independent of $T$; balancing the terms gives $\mathcal O(1/\sqrt{T})$ \citep{wangCRAFTConflictResolvedAggregation2026}. FedPCD's theorem proves $\mathcal O(1/T)$ for the same measure under a constant admissible server step. CRAFT assumes bounded local gradients and updates, while FedPCD uses a uniform angular margin and relative drift. These are comparisons of upper bounds under their respective conditions, not a lower bound on CRAFT's attainable rate. In FedPCD's proof, $\lVert d_t\rVert=\lVert\bar g_t\rVert\leq K(1+r)S_t$ allows the smoothness remainder to be absorbed into descent. Therefore, the rate does not follow from inequality constraints alone.

\section{Expanded Results}
\label{app:expanded-results}

\subsection{Additional $\alpha$ values}\
\label{app:all-alpha}

Figs.~\ref{fig:client-distribution-alpha0.1}--\ref{fig:communication-rounds-convergence-alpha1} show the client distributions, comparison plots, and convergence results for $\alpha=0.1$ and $\alpha=1$. All figures use $\tau=0.02$ and $H=10$, consistent with Figs.~\ref{fig:fairness-frontier-alpha0.01}--\ref{fig:communication-rounds-convergence-alpha0.01} for $\alpha=0.01$.

\begin{figure}[h]
    \centering
    \includegraphics[width=\linewidth]{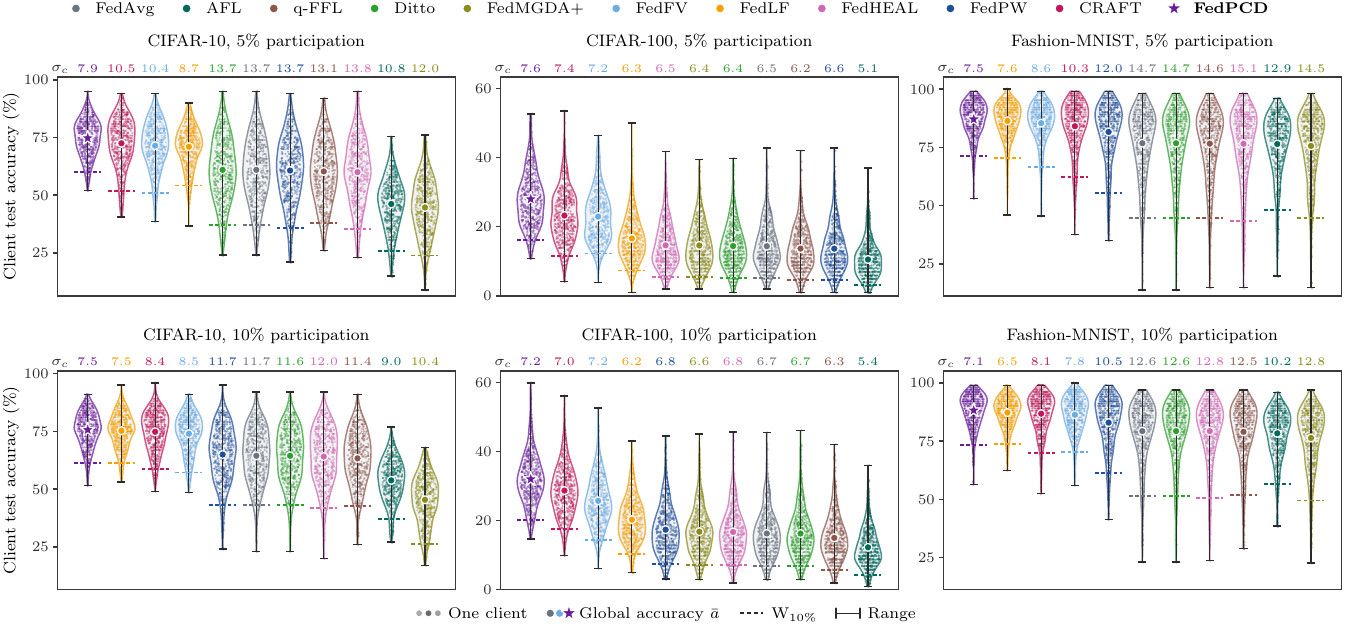}
    \caption{Client test-accuracy distributions at $\alpha=0.1$ across datasets and participation rates. Methods are ordered by global accuracy with each violin representing the kernel density of client accuracies pooled over the 5 seeds.}
    \label{fig:client-distribution-alpha0.1}
\end{figure}

\begin{figure}[h]
    \centering
    \includegraphics[width=\linewidth]{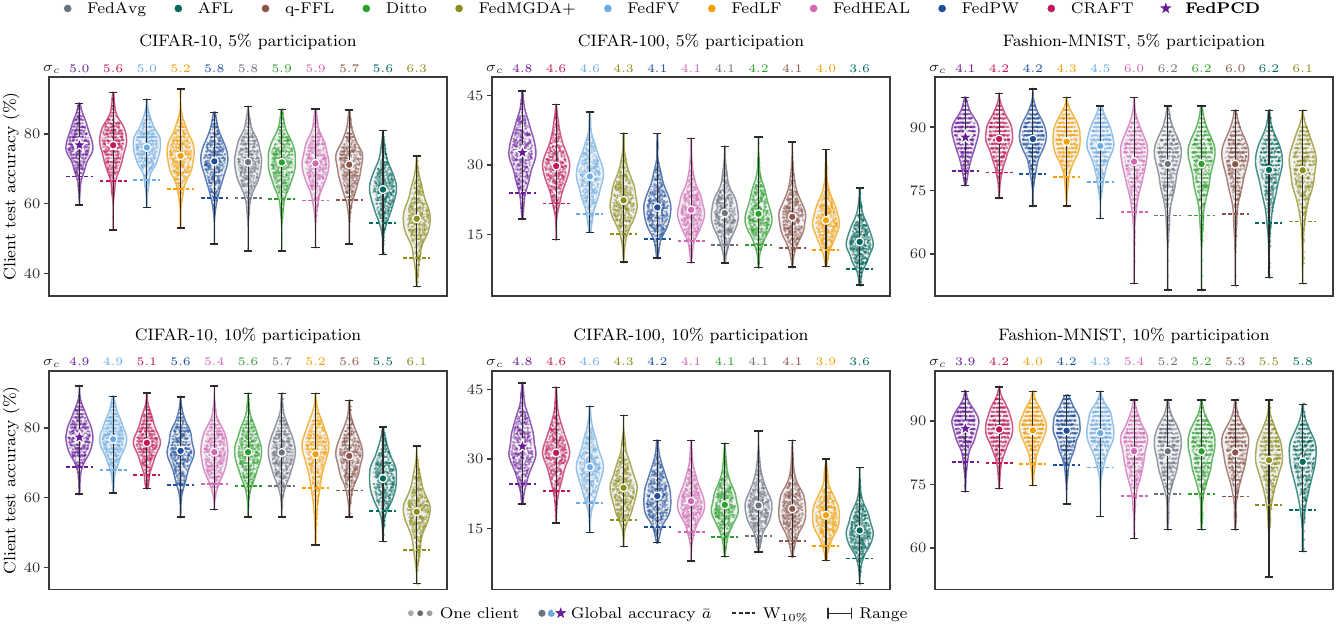}
    \caption{Client test-accuracy distributions at $\alpha=1$ across datasets and participation rates. Methods are ordered by global accuracy with each violin representing the kernel density of client accuracies pooled over the 5 seeds.}
    \label{fig:client-distribution-alpha1}
\end{figure}

\begin{figure}[h]
    \centering
    \includegraphics[width=0.85\linewidth]{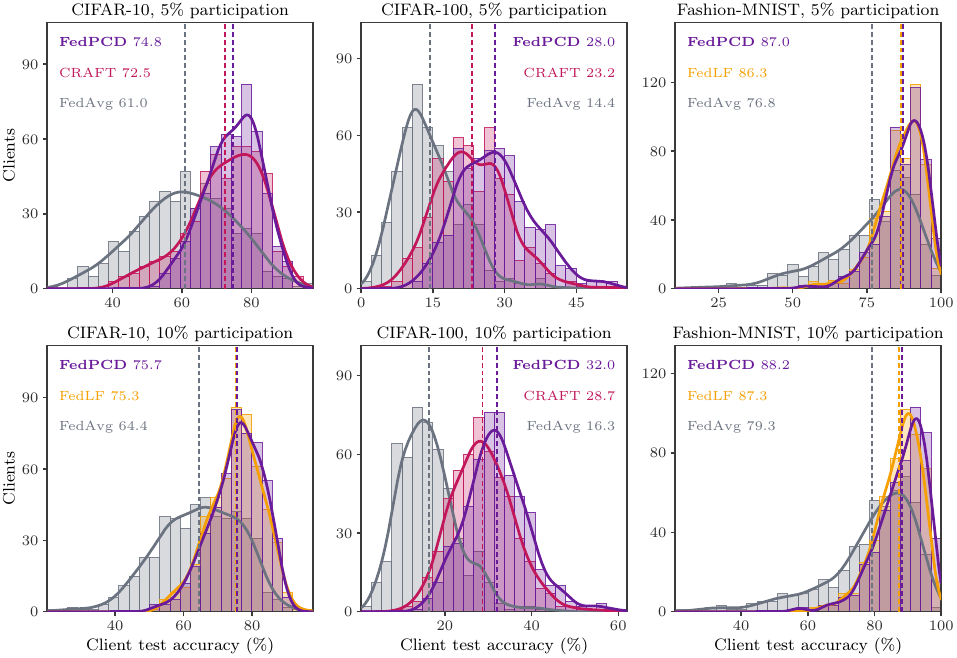}
    \caption{Client test accuracy distributions for FedPCD, FedAvg, and the strongest non-FedPCD baseline at $\alpha = 0.1$. The distribution means are shown as dashed lines and their values are reported per panel.}
    \label{fig:client-histogram-alpha0.1}
\end{figure}

\begin{figure}[h]
    \centering
    \includegraphics[width=0.85\linewidth]{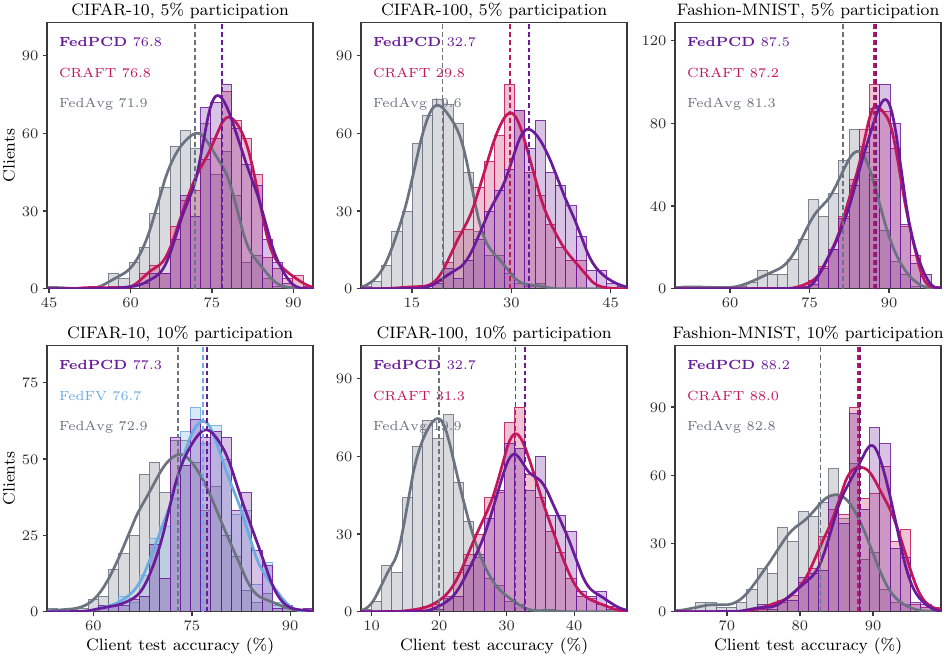}
    \caption{Client test accuracy distributions for FedPCD, FedAvg, and the strongest non-FedPCD baseline at $\alpha = 1$. The distribution means are shown as dashed lines and their values are reported per panel.}
    \label{fig:client-histogram-alpha1}
\end{figure}

\begin{figure}[h]
    \centering
    \includegraphics[width=0.75\linewidth]{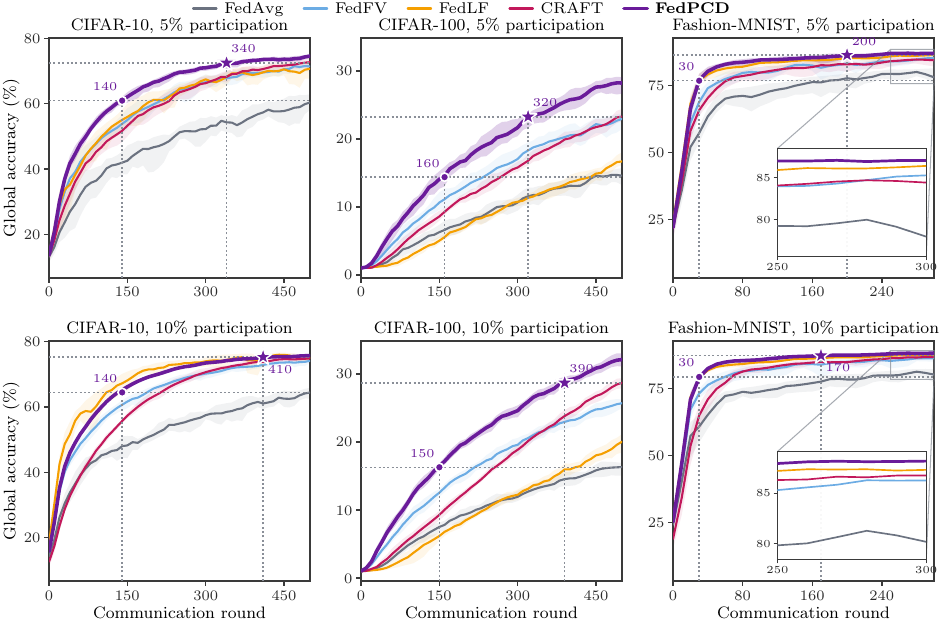}
    \caption{Global accuracy against communication rounds at $\alpha=0.1$. Dots mark when FedPCD reaches FedAvg's final accuracy, while stars mark when it reaches the strongest baseline's final accuracy. When the latter is not reached, a hollow ring marks when that baseline reaches FedAvg's final accuracy.}
    \label{fig:communication-rounds-convergence-alpha0.1}
\end{figure}

\begin{figure}[h]
    \centering
    \includegraphics[width=0.75\linewidth]{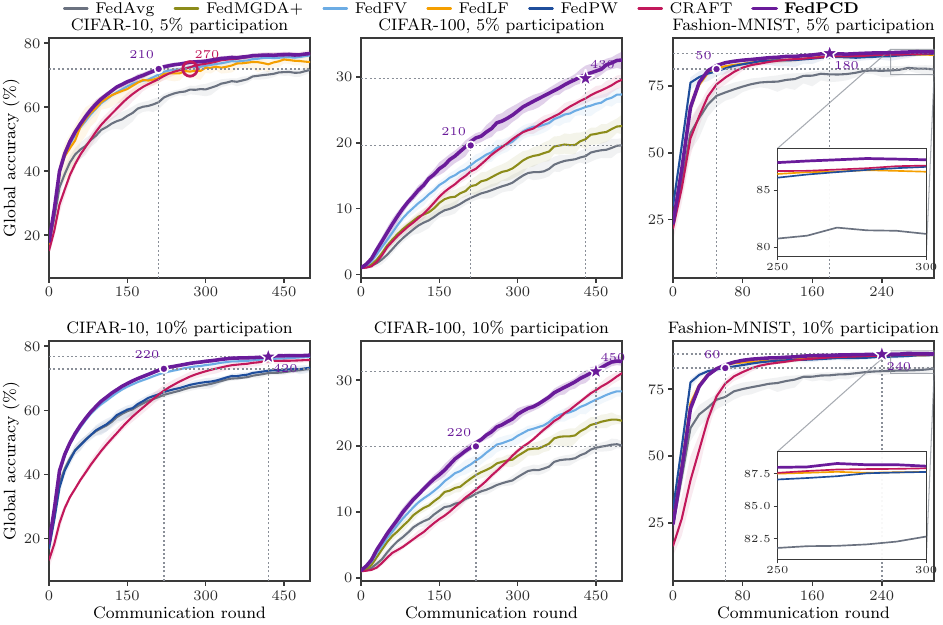}
    \caption{Global accuracy against communication rounds at $\alpha=1$. Dots mark when FedPCD reaches FedAvg's final accuracy, while stars mark when it reaches the strongest baseline's final accuracy. When the latter is not reached, a hollow ring marks when that baseline reaches FedAvg's final accuracy.}
    \label{fig:communication-rounds-convergence-alpha1}
\end{figure}

\clearpage

\subsection{Fairness Study}
\label{app:fairness}

Figs.~\ref{fig:fairness-frontier-alpha0.1} and~\ref{fig:fairness-frontier-alpha1} show the accuracy against the standard deviation for $\alpha = 0.1$ and $1$. We also provide Table~\ref{tab:jain} as a complementary fairness measure. For a vector of client accuracies $x$ over $N$ clients, $J(x)=(\sum_i x_i)^2 / (N\sum_i x_i^2)$. It equals 1 when every client reaches the same accuracy and is unchanged when every client is scaled by the same factor, so unlike $\sigma_c$ it does not grow with the accuracy it is measured around.

\begin{figure}[h]
    \centering
    \includegraphics[width=0.65\linewidth]{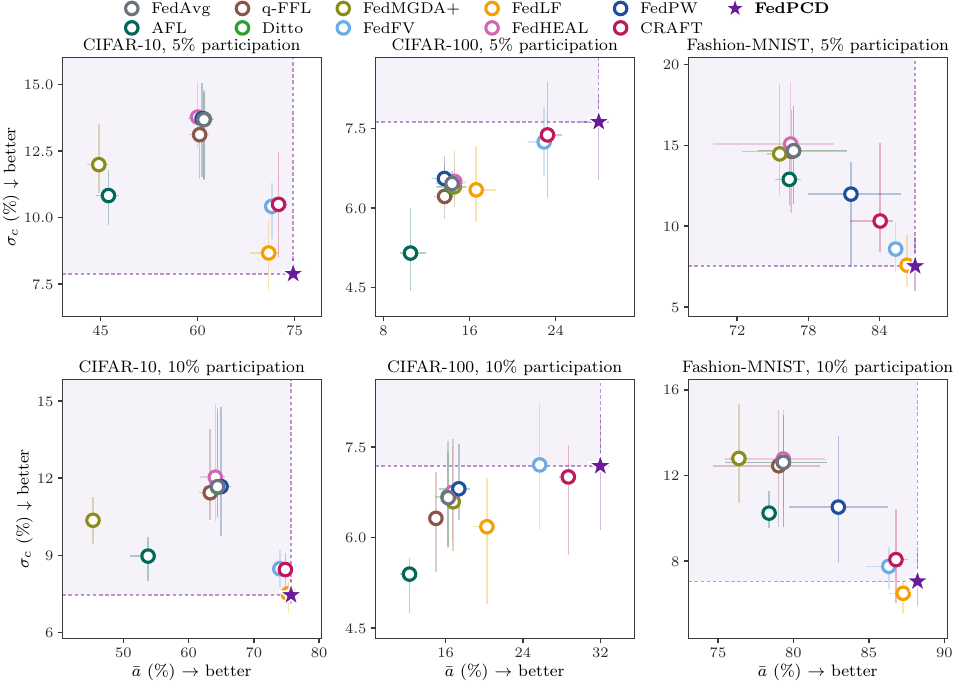}
    \caption{Global accuracy against client-accuracy spread at $\alpha=0.1$ for both participation rates. Higher accuracy and lower $\sigma_c$ are preferred. Crosses indicate variation across the 5 seeds.}
    \label{fig:fairness-frontier-alpha0.1}
\end{figure}

\begin{figure}[h]
    \centering
    \includegraphics[width=0.65\linewidth]{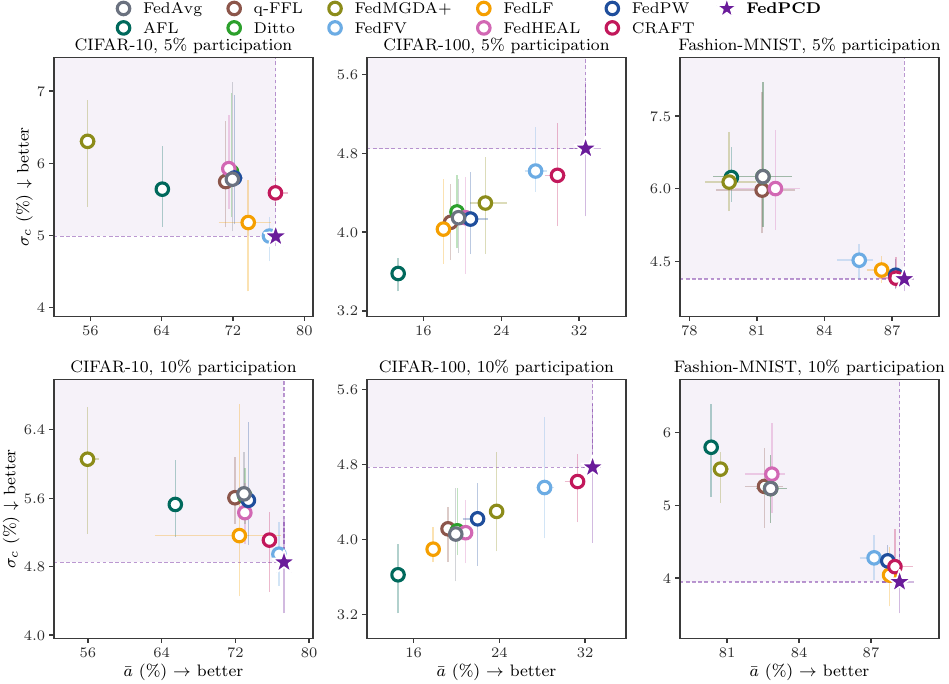}
    \caption{Global accuracy against client-accuracy spread at $\alpha=1$ for both participation rates. Higher accuracy and lower $\sigma_c$ are preferred. Crosses indicate variation across the 5 seeds.}
    \label{fig:fairness-frontier-alpha1}
\end{figure}

\begin{table}[!t]
\centering
\caption{Jain's fairness index over the per-client accuracies, under the same settings as Table~\ref{tab:main_results}. Higher is better. Within each dataset block, bold marks the best overall value in a column and underline marks the best value from the competing method group, such that the strongest FedPCD configuration and strongest baseline are both highlighted.}
\label{tab:jain}

\setlength{\tabcolsep}{5pt}
\renewcommand{\arraystretch}{1.0}
\footnotesize
\begin{tabular}{lcccccc}
\toprule
Method & \multicolumn{3}{c}{5\% participation} & \multicolumn{3}{c}{10\% participation} \\
\cmidrule(lr){2-4} \cmidrule(lr){5-7}
 & $\alpha=0.01$ & $\alpha=0.1$ & $\alpha=1$ & $\alpha=0.01$ & $\alpha=0.1$ & $\alpha=1$ \\
\midrule
\multicolumn{7}{@{}l}{\textbf{CIFAR-10}} \\
\addlinespace[1pt]
FedAvg & 0.8574 & 0.9516 & 0.9935 & 0.9256 & 0.9674 & 0.9940 \\
AFL & 0.6285 & 0.9476 & 0.9923 & 0.9346 & 0.9727 & 0.9929 \\
q-FFL & 0.8958 & 0.9545 & 0.9935 & 0.9320 & 0.9678 & 0.9940 \\
Ditto & 0.8593 & 0.9514 & 0.9933 & 0.9258 & 0.9676 & 0.9941 \\
FedMGDA+ & 0.8860 & 0.9322 & 0.9872 & 0.8974 & 0.9502 & 0.9883 \\
FedFV & 0.9393 & 0.9791 & \underline{0.9957} & 0.9754 & 0.9870 & \underline{0.9958} \\
FedLF & \underline{0.9527} & \underline{0.9850} & 0.9950 & \underline{0.9858} & \underline{0.9901} & 0.9945 \\
FedHEAL & 0.8458 & 0.9494 & 0.9931 & 0.9226 & 0.9651 & 0.9945 \\
FedPW & 0.8469 & 0.9508 & 0.9935 & 0.9069 & 0.9679 & 0.9942 \\
CRAFT & 0.9421 & 0.9789 & 0.9947 & 0.9739 & 0.9874 & 0.9954 \\
\cmidrule(lr){1-7}
FedPCD$_{\tau=0.005}$ & \textbf{0.9803} & 0.9877 & 0.9959 & 0.9847 & 0.9893 & 0.9960 \\
FedPCD$_{\tau=0.01}$ & 0.9797 & 0.9882 & 0.9957 & \textbf{0.9860} & 0.9895 & 0.9958 \\
FedPCD$_{\tau=0.02}$ & 0.9795 & \textbf{0.9890} & 0.9958 & 0.9850 & \textbf{0.9904} & 0.9961 \\
FedPCD$_{\tau=0.05}$ & 0.9680 & 0.9875 & 0.9960 & 0.9805 & 0.9899 & 0.9961 \\
FedPCD$_{\tau=0.075}$ & 0.9588 & 0.9872 & \textbf{0.9961} & 0.9755 & 0.9894 & \textbf{0.9964} \\
FedPCD$_{\tau=0.1}$ & 0.9610 & 0.9855 & \textbf{0.9961} & 0.9742 & 0.9887 & 0.9962 \\
\midrule
\multicolumn{7}{@{}l}{\textbf{CIFAR-100}} \\
\addlinespace[1pt]
FedAvg & 0.4750 & 0.8302 & 0.9570 & 0.5632 & 0.8548 & 0.9600 \\
AFL & 0.4131 & 0.8058 & 0.9333 & 0.5390 & 0.8372 & 0.9415 \\
q-FFL & 0.4722 & 0.8280 & 0.9545 & 0.5387 & 0.8485 & 0.9561 \\
Ditto & 0.4696 & 0.8339 & 0.9552 & 0.5700 & 0.8538 & 0.9599 \\
FedMGDA+ & 0.4411 & 0.8371 & 0.9639 & 0.5555 & 0.8644 & 0.9679 \\
FedFV & 0.6250 & \underline{0.9091} & 0.9724 & 0.7316 & 0.9268 & 0.9744 \\
FedLF & 0.4119 & 0.8717 & 0.9518 & 0.6285 & 0.9147 & 0.9544 \\
FedHEAL & 0.4408 & 0.8336 & 0.9595 & 0.5679 & 0.8578 & 0.9631 \\
FedPW & 0.4451 & 0.8106 & 0.9620 & 0.5687 & 0.8645 & 0.9639 \\
CRAFT & \underline{0.6495} & 0.9072 & \underline{0.9768} & \underline{0.7939} & \underline{0.9432} & \underline{0.9787} \\
\cmidrule(lr){1-7}
FedPCD$_{\tau=0.005}$ & 0.6913 & 0.9269 & 0.9771 & 0.7946 & 0.9454 & 0.9753 \\
FedPCD$_{\tau=0.01}$ & 0.6968 & 0.9267 & 0.9762 & 0.8005 & 0.9477 & 0.9768 \\
FedPCD$_{\tau=0.02}$ & \textbf{0.7037} & 0.9309 & 0.9781 & \textbf{0.8067} & 0.9517 & 0.9790 \\
FedPCD$_{\tau=0.05}$ & 0.6759 & \textbf{0.9343} & \textbf{0.9815} & 0.7997 & \textbf{0.9518} & 0.9821 \\
FedPCD$_{\tau=0.075}$ & 0.6622 & 0.9266 & 0.9800 & 0.7903 & 0.9499 & \textbf{0.9825} \\
FedPCD$_{\tau=0.1}$ & 0.6297 & 0.9170 & 0.9793 & 0.7747 & 0.9430 & 0.9818 \\
\midrule
\multicolumn{7}{@{}l}{\textbf{Fashion-MNIST}} \\
\addlinespace[1pt]
FedAvg & 0.9390 & 0.9639 & 0.9939 & 0.9388 & 0.9741 & 0.9960 \\
AFL & 0.9498 & 0.9721 & 0.9939 & 0.9689 & 0.9831 & 0.9948 \\
q-FFL & 0.9407 & 0.9635 & 0.9944 & 0.9375 & 0.9747 & 0.9959 \\
Ditto & 0.9390 & 0.9639 & 0.9939 & 0.9388 & 0.9741 & 0.9960 \\
FedMGDA+ & 0.9077 & 0.9637 & 0.9940 & 0.9239 & 0.9722 & 0.9954 \\
FedFV & 0.9727 & 0.9899 & 0.9972 & 0.9757 & 0.9919 & 0.9976 \\
FedLF & \underline{0.9840} & \underline{0.9921} & 0.9975 & \textbf{0.9885} & \textbf{0.9944} & \underline{0.9979} \\
FedHEAL & 0.9334 & 0.9605 & 0.9946 & 0.9372 & 0.9739 & 0.9957 \\
FedPW & 0.9353 & 0.9777 & \underline{0.9977} & 0.9625 & 0.9832 & 0.9977 \\
CRAFT & 0.9523 & 0.9841 & \underline{0.9977} & 0.9718 & 0.9912 & 0.9978 \\
\cmidrule(lr){1-7}
FedPCD$_{\tau=0.005}$ & \textbf{0.9860} & 0.9927 & 0.9979 & 0.9841 & \underline{0.9937} & 0.9980 \\
FedPCD$_{\tau=0.01}$ & 0.9842 & \textbf{0.9928} & \textbf{0.9980} & \underline{0.9856} & 0.9935 & \textbf{0.9981} \\
FedPCD$_{\tau=0.02}$ & 0.9855 & 0.9924 & 0.9978 & 0.9833 & 0.9936 & 0.9980 \\
FedPCD$_{\tau=0.05}$ & 0.9805 & 0.9922 & 0.9977 & 0.9844 & 0.9928 & 0.9978 \\
FedPCD$_{\tau=0.075}$ & 0.9815 & 0.9913 & 0.9974 & 0.9852 & 0.9929 & 0.9977 \\
FedPCD$_{\tau=0.1}$ & 0.9792 & 0.9909 & 0.9973 & 0.9828 & 0.9929 & 0.9977 \\
\bottomrule
\end{tabular}
\end{table}

\clearpage
\section{Ablations}
\label{app:ablations}

\subsection{Component Ablation}
To assess the contribution of the individual components of FedPCD, we experiment with removing the history mechanism, the secondary progress constraints by setting $\tau=0$, the normalization step, and the primary objective independently. All of these experiments are conducted with 10\% participation rate with the CIFAR-10 dataset averaged over 5 seeds. For nonzero-$\tau$ variants, we use $\tau=(0.01,0.02,0.075)$ for $\alpha=(0.01,0.1,1)$, respectively, which are the best $\tau$ values in their respective $\alpha$ from Table~\ref{tab:main_results}. For the history-free ablation, we report the best tested values,
$\tau=(0.1,0.1,0.1)$ for $\alpha=(0.01,0.1,1)$, respectively. Because removing history reduces the number of secondary constraints, the QP can accommodate larger progress thresholds, which explains why the preferred $\tau$ values are higher than those used with history. The results of these experiments are reported in Table~\ref{tab:component-ablation}. The history mechanism has the largest effect, particularly under stronger heterogeneity. Removing history reduces FedPCD from $74.84\%$ to $59.26\%$ global accuracy at $\alpha=0.01$, with a corresponding drop in worst-$10\%$ accuracy. This supports the motivation for retaining recently observed client directions when only a small fraction of clients participate in each round. Setting $\tau=0$ also consistently reduces performance relative to the complete method, showing that enforcing positive client progress provides an additional benefit beyond the history mechanism itself. Removing normalization leads to smaller but consistent degradations in global and tail-client accuracy, while removing the primary direction causes a larger loss, confirming the importance of anchoring the constrained solution to the current federated aggregate. Overall, the complete formulation provides the strongest balance across the three heterogeneity levels.

\begin{table}[h]
\centering
\caption{Component ablation on CIFAR-10 at 10\% participation.}
\label{tab:component-ablation}
\setlength{\tabcolsep}{2.8pt}
\renewcommand{\arraystretch}{1.1}

\begin{tabular}{lccccccccc}
\toprule
Method & \multicolumn{3}{c}{$\alpha=0.01$} & \multicolumn{3}{c}{$\alpha=0.1$} & \multicolumn{3}{c}{$\alpha=1$} \\
\cmidrule(lr){2-4}\cmidrule(lr){5-7}\cmidrule(lr){8-10}
& $\bar a\uparrow$ & $\bar\sigma_c\downarrow$ & $W_{10\%}\uparrow$
& $\bar a\uparrow$ & $\bar\sigma_c\downarrow$ & $W_{10\%}\uparrow$
& $\bar a\uparrow$ & $\bar\sigma_c\downarrow$ & $W_{10\%}\uparrow$ \\
\midrule
FedAvg & 57.47 & 15.93 & 29.93 & 64.43 & 11.67 & 43.27 & 72.94 & 5.65 & 63.28 \\
\midrule
FedPCD & \textbf{74.84} & \textbf{8.89} & \textbf{58.03} & \textbf{75.69} & \textbf{7.45} & \textbf{61.28} & \textbf{77.51} & 4.68 & \textbf{69.24} \\
FedPCD w/o history & 59.26 & 16.24 & 29.65 & 65.50 & 11.64 & 44.13 & 73.00 & 5.69 & 63.16 \\
FedPCD$_{\tau=0}$ & 74.31 & 9.15 & 57.55 & 74.88 & 7.90 & 59.91 & 76.96 & 5.12 & 67.85 \\
FedPCD w/o normalization & 72.99 & 9.40 & 54.44 & 75.08 & 7.46 & 60.85 & 76.96 & \textbf{4.66} & 68.81 \\
FedPCD w/o primary & 70.75 & 12.14 & 46.86 & 73.40 & 8.12 & 57.60 & 76.33 & 4.91 & 67.44 \\
\bottomrule
\end{tabular}
\end{table}

\subsection{History window sweep}

To examine how much historical client information is needed under partial participation, we vary the history horizon $H$ from 0 to 10 and report global accuracy, worst-$10\%$ client accuracy, and the number of active secondary constraints in Fig.~\ref{fig:history-sweep}. Increasing $H$ consistently improves both global and tail-client performance, with the largest gains occurring over the first few additional rounds of history. On CIFAR-10, for instance, the $5\%$ participation setting increases from below $50\%$ to over $70\%$ global accuracy and from approximately $20\%$ to above $50\%$ worst-$10\%$ accuracy as the history window is expanded. CIFAR-100 exhibits a similarly strong effect, with both metrics continuing to improve across tested range. In contrast, Fashion-MNIST reaches most of its global accuracy improvement by approximately $H=3$--$5$, after which the benefit of additional history becomes comparatively small.

The bottom row of Fig.~\ref{fig:history-sweep} shows that these gains coincide with a steady increase in the number of active secondary constraints, confirming that a larger history window allows a broader set of recently observed clients to influence the constrained update. The effect is stronger at $10\%$ participation, where more secondary objectives are taken into account over the same history horizon. Importantly, performance eventually exhibits diminishing returns even though the number of constraints continues to increase. This is particularly visible on CIFAR-10 and Fashion-MNIST, suggesting that once sufficient client coverage has been obtained, retaining increasingly old directions provides limited additional benefit. Overall, the sweep supports the use of history as a mechanism for compensating for partial participation, while also showing that excessively long history windows are not necessary to obtain most of the improvement.

\begin{figure*}[!h]
    \centering
    \includegraphics[width=1\linewidth]{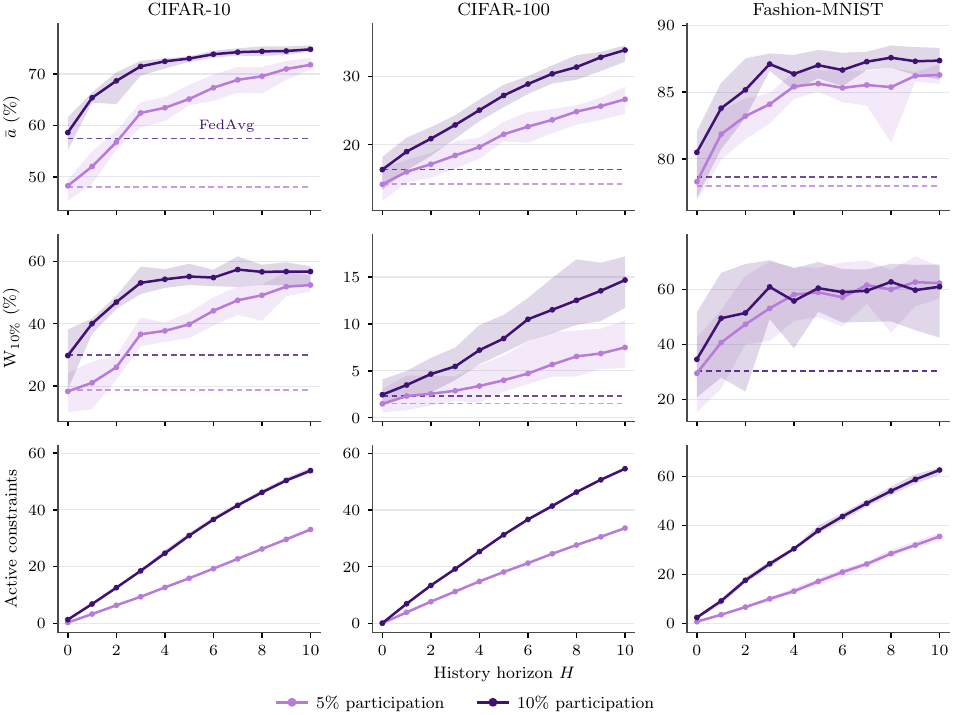}
    \caption{Impact of increasing history window size $H$ on the global accuracy, the worst-10\% across clients, and the number of active secondary constraints. This is reported across all three datasets, with $\tau = 0.02$ and $\alpha = 0.01$.}
    \label{fig:history-sweep}
\end{figure*}

\subsection{Normalization Study}

The main FedPCD formulation uses instantaneous $\ell_2$ normalization of the current client directions, whereas the original PCD formulation also considers normalization based on exponential moving averages \citet{varamNotAllObjectives2026}. Therefore, we compare instantaneous normalization against EMA normalization with $\beta\in\{0.5,0.8,0.9\}$ to determine whether maintaining normalization statistics across rounds is beneficial in the federated setting. As shown in Table~\ref{tab:normalization-ablation}, instantaneous $\ell_2$ normalization achieves the best global accuracy at all three heterogeneity levels and gives the strongest results for most of the remaining metrics. The main exception occurs at $\alpha=0.1$, where EMA with $\beta=0.8$ slightly improves the client standard deviation from $7.45\%$ to $7.44\%$ and the worst-$10\%$ accuracy from $61.28\%$ to $61.83\%$. These differences are small and do not persist across the other settings. The basic $\ell_2$ normalization provides better performance without requiring client-specific normalization statistics to be maintained across intermittent participation, supporting its use in the main FedPCD formulation. This is to be expected since the normalization statistics for the secondary objectives should not be expected to be updated frequently due to partial participation.

\begin{table}[!h]
\centering
\caption{Normalization ablation on CIFAR-10 at 10\% participation and round 500. Results are averaged over five seeds with $H=10$. We use $\tau=(0.01,0.02,0.075)$ for $\alpha=(0.01,0.1,1)$, respectively.}
\label{tab:normalization-ablation}
\setlength{\tabcolsep}{2.8pt}
\renewcommand{\arraystretch}{1.1}
\begin{tabular}{lccccccccc}
\toprule
Normalization & \multicolumn{3}{c}{$\alpha=0.01$} & \multicolumn{3}{c}{$\alpha=0.1$} & \multicolumn{3}{c}{$\alpha=1$} \\
\cmidrule(lr){2-4}\cmidrule(lr){5-7}\cmidrule(lr){8-10}
& $\bar a\uparrow$ & $\bar\sigma_c\downarrow$ & $W_{10\%}\uparrow$
& $\bar a\uparrow$ & $\bar\sigma_c\downarrow$ & $W_{10\%}\uparrow$
& $\bar a\uparrow$ & $\bar\sigma_c\downarrow$ & $W_{10\%}\uparrow$ \\
\midrule
Instantaneous $\ell_2$ & \textbf{74.84} & \textbf{8.89} & \textbf{58.03} & \textbf{75.69} & 7.45 & 61.28 & \textbf{77.51} & \textbf{4.68} & \textbf{69.24} \\
Scalar EMA ($\beta=0.5$) & 74.77 & 8.98 & 57.61 & 75.54 & 7.66 & 60.68 & 77.46 & 4.83 & 68.77 \\
Scalar EMA ($\beta=0.8$) & 74.69 & 9.17 & 56.77 & \textbf{75.69} & \textbf{7.44} & \textbf{61.83} & 77.35 & 4.82 & 68.73 \\
Scalar EMA ($\beta=0.9$) & 74.81 & 9.07 & 57.30 & 75.51 & 7.68 & 60.66 & 77.29 & 4.78 & 68.90 \\
\bottomrule
\end{tabular}
\end{table}

\end{document}

%% file: math_commands.tex
\usepackage{amsmath,amsfonts,bm}

\def\eqref#1{equation~\ref{#1}}

\def\1{\bm{1}}

\DeclareMathAlphabet{\mathsfit}{\encodingdefault}{\sfdefault}{m}{sl}
\SetMathAlphabet{\mathsfit}{bold}{\encodingdefault}{\sfdefault}{bx}{n}



%% file: figures/fig_taxonomy.tex
\definecolor{taxA}{RGB}{42,127,98}%
\definecolor{taxB}{RGB}{47,109,181}%
\definecolor{taxC}{RGB}{176,124,26}%
\definecolor{taxD}{RGB}{111,78,156}%
\definecolor{taxE}{RGB}{106,27,154}%
\definecolor{taxInk}{RGB}{70,70,70}%

\ifdefined\taxRootW\else
  \newlength{\taxRootW}\newlength{\taxGapX}\newlength{\taxGapY}\newlength{\taxPad}\newlength{\taxBul}
  \newlength{\taxLoneW}\newlength{\taxLtwoW}\newlength{\taxLeafW}\newlength{\taxBoxH}
  \newlength{\taxXone}\newlength{\taxXtwo}\newlength{\taxXleaf}
\fi
\setlength{\taxRootW}{0.50cm}%
\setlength{\taxGapX}{0.24cm}%
\setlength{\taxGapY}{0.17cm}%
\setlength{\taxPad}{3.5pt}%
\setlength{\taxBul}{6pt}%
\setlength{\taxLoneW}{4.85cm}%
\setlength{\taxLtwoW}{2.34cm}%
\setlength{\taxBoxH}{0.90cm}%
\setlength{\taxXone}{\dimexpr\taxRootW+\taxGapX\relax}%
\setlength{\taxXtwo}{\dimexpr\taxXone+\taxLoneW+\taxGapX\relax}%
\setlength{\taxXleaf}{\dimexpr\taxXtwo+\taxLtwoW+\taxGapX\relax}%
\setlength{\taxLeafW}{\dimexpr\linewidth-\taxXleaf-1pt\relax}%

\newcommand{\taxbox}[3]{\parbox[c][\dimexpr#1-2\taxPad\relax][c]{\dimexpr#2-2\taxPad\relax}{\centering\footnotesize#3\par}}
\newcommand{\taxfam}[3]{\taxbox{\taxBoxH}{\taxLoneW}{{\bfseries\color{#1}#2}\\{\scriptsize\itshape\color{taxInk}#3}}}
\newcommand{\taxsub}[3]{\taxbox{\taxBoxH}{\taxLtwoW}{{\bfseries\color{#1}#2}\\{\scriptsize\itshape\color{taxInk}#3}}}
\newcommand{\taxleaf}[1]{\parbox{\dimexpr\taxLeafW-2\taxPad\relax}{\scriptsize\setlength{\parindent}{0pt}\setlength{\parskip}{0pt}\raggedright#1\par}}
\newcommand{\taxm}[3]{\mbox{{\color{#1}\textbullet}~\textbf{#2} \citep{#3}}}
\newcommand{\taxsep}{\hspace{0.8em}}
\newcommand{\taxitem}[2]{\par\noindent\hangindent=\taxBul\hangafter=1\makebox[\taxBul][l]{\textcolor{#1}{\textbullet}}#2\par}

\begin{tikzpicture}[
  every node/.style={inner sep=\taxPad, outer sep=0pt},
  box/.style={draw=#1, line width=0.5pt, rounded corners=2pt, fill=#1!5},
  fam/.style={box=#1, fill=#1!12},
  leaf/.style={box=#1, fill=white, minimum height=\taxBoxH},
  ours/.style={line width=1.1pt},
  edge/.style={draw=black!60, line width=0.6pt},
  arr/.style={edge, -{Stealth[length=3.6pt, width=3.2pt]}},
]

\node[leaf=taxA, anchor=north west] (l1) at (\taxXleaf,0) {\taxleaf{%
  \taxm{taxA}{AFL}{mohriAgnosticFederatedLearning2019}\taxsep
  \taxm{taxA}{q-FFL}{liFairResourceAllocation2019}\taxsep
  \taxm{taxA}{TERM}{liTiltedEmpiricalRisk2020}\taxsep
  \taxm{taxA}{FedFa}{huangFairnessAccuracyHorizontal2022}}};
\node[leaf=taxB, anchor=north west] (l2) at ($(l1.south west)+(0,-\taxGapY)$) {\taxleaf{%
  \taxm{taxB}{FedHEAL}{chenFairFederatedLearning2024}\taxsep
  \taxm{taxB}{FedPW}{guanRethinkingFairFederated2025}}};
\node[leaf=taxC, anchor=north west] (l3) at ($(l2.south west)+(0,-\taxGapY)$) {\taxleaf{%
  \taxm{taxC}{FedFV$^\dagger$}{wangFederatedLearningFair2021}\taxsep
  \taxm{taxC}{CRAFT}{wangCRAFTConflictResolvedAggregation2026}}};
\node[leaf=taxD, anchor=north west] (l4) at ($(l3.south west)+(0,-\taxGapY)$) {\taxleaf{%
  \taxm{taxD}{FedMGDA+}{huFederatedLearningMeets2022}\taxsep
  \taxm{taxD}{AdaFed}{hamidiAdaFedFairFederated2023}\taxsep
  \taxm{taxD}{FedMDFG$^\dagger$}{panFedMDFGFederatedLearning2023}\taxsep
  \taxm{taxD}{FedLF$^\dagger$}{panFedLFLayerWiseFair2024}\taxsep
  \taxm{taxD}{DQN-Fed}{hamidiDistributedQuasiNewtonMethod2025}\taxsep
  \taxm{taxD}{FedMC+$^\dagger$}{shenMultiobjectiveFederatedLearning2025}}};
\node[leaf=taxE, ours, fill=taxE!4, anchor=north west] (l5) at ($(l4.south west)+(0,-\taxGapY)$) {\taxleaf{%
  {\bfseries\color{taxE}FedPCD (ours)}, built on PCD \citep{varamNotAllObjectives2026}\par
  \taxitem{taxE}{\textbf{Primary:} the FedAvg-like aggregate $\tilde{\bar g}_t$}
  \taxitem{taxE}{\textbf{Secondary:} client progress constraints $\tilde g_{i,t}^{\top} d \ge \tau\lVert\tilde g_{i,t}\rVert^2$}
  \taxitem{taxE}{\textbf{Minimal correction:} $d_t^\star=\tilde{\bar g}_t$ if all constraints hold}
  \taxitem{taxE}{\textbf{History-aware}$^\dagger$: cached $\tilde h_{i,t}$ under partial participation}
  \taxitem{taxE}{$\mathcal{O}(1/T)$ convergence rate under full participation}}};

\node[fam=taxD, anchor=west] (s4) at (\taxXtwo,0 |- l4.west) {\taxsub{taxD}{Symmetric}{common descent}};
\node[fam=taxE, ours, anchor=west] (s5) at (\taxXtwo,0 |- l5.west) {\taxsub{taxE}{Hierarchical}{(ours)}};

\node[fam=taxA, anchor=west] (f1) at (\taxXone,0 |- l1.west) {\taxfam{taxA}{Loss / Client Reweighting}{Reweight losses across clients}};
\node[fam=taxB, anchor=west] (f2) at (\taxXone,0 |- l2.west) {\taxfam{taxB}{Parameter-wise Aggregation}{Filter or rescale parameter updates}};
\node[fam=taxC, anchor=west] (f3) at (\taxXone,0 |- l3.west) {\taxfam{taxC}{Conflict-Aware Correction}{Reduce client update conflicts}};
\coordinate (m4) at ($(l4.north west)!0.5!(l5.south west)$);
\node[fam=taxD, anchor=west] (f4) at (\taxXone,0 |- m4) {\taxfam{taxD}{Multi-Objective FL}{Clients as separate objectives}};

\path let \p1=(l1.north west), \p2=(l5.south west) in
  node[draw=black!65, line width=0.6pt, rounded corners=2pt, fill=black!5, inner sep=0pt,
       minimum width=\taxRootW, minimum height=\y1-\y2, anchor=north west] (root) at (0,\y1)
       {\rotatebox{90}{\small\bfseries Fairness-Aware Federated Learning}};

\coordinate (spine) at ($(root.east)+(0.08,0)$);
\draw[edge] (root.east) -- (spine);
\draw[edge] (spine |- f1.west) -- (spine |- f4.west);
\foreach \n in {f1,f2,f3,f4} {\draw[arr] (spine |- \n.west) -- (\n.west);}
\foreach \n/\l in {f1/l1,f2/l2,f3/l3} {\draw[arr] (\n.east) -- (\l.west);}
\coordinate (spine2) at ($(f4.east)+(0.08,0)$);
\draw[edge, draw=taxD] (f4.east) -- (spine2);
\draw[edge, draw=taxD] (spine2 |- s4.west) -- (spine2 |- s5.west);
\draw[arr, draw=taxD] (spine2 |- s4.west) -- (s4.west);
\draw[arr, draw=taxE] (spine2 |- s5.west) -- (s5.west);
\draw[arr, draw=taxD] (s4.east) -- (l4.west);
\draw[arr, draw=taxE] (s5.east) -- (l5.west);

\node[anchor=north east, inner sep=0pt, font=\scriptsize, text=taxInk]
  at ($(l5.south east)+(0,-0.14)$) {$^\dagger$ uses historical client information under partial participation};

\end{tikzpicture}

%% file: FedPCD.bib
@inproceedings{chenFairFederatedLearning2024,
  title = {Fair {{Federated Learning Under Domain Skew}} with {{Local Consistency}} and {{Domain Diversity}}},
  booktitle = {2024 {{IEEE}}/{{CVF Conference}} on {{Computer Vision}} and {{Pattern Recognition}} ({{CVPR}})},
  author = {Chen, Yuhang and Huang, Wenke and Ye, Mang},
  year = 2024,
  month = jun,
  pages = {12077--12086},
  issn = {2575-7075},
  doi = {10.1109/CVPR52733.2024.01148},
  urldate = {2026-08-18},

}

@article{desideriMultiplegradientDescentAlgorithm2012,
  title = {Multiple-Gradient Descent Algorithm ({{MGDA}}) for Multiobjective Optimization},
  author = {D{\'e}sid{\'e}ri, Jean-Antoine},
  year = 2012,
  month = mar,
  journal = {Comptes Rendus Mathematique},
  volume = {350},
  number = {5},
  pages = {313--318},
  issn = {1631-073X},
  doi = {10.1016/j.crma.2012.03.014},
  urldate = {2026-08-23},
  
}

@inproceedings{guanRethinkingFairFederated2025,
  title = {Rethinking {{Fair Federated Learning}} from {{Parameter}} and {{Client View}}},
  booktitle = {The {{Thirty-ninth Annual Conference}} on {{Neural Information Processing Systems}}},
  author = {Guan, Kaiqi and Huang, Wenke and Guo, Xianda and Yuan, Yueyang and Yang, Bin and Ye, Mang},
  year = 2025,
  month = oct,
  urldate = {2026-08-18},
  langid = {english},
  
}

@article{hamidiAdaFedFairFederated2023,
  title = {{{AdaFed}}: {{Fair Federated Learning}} via {{Adaptive Common Descent Direction}}},
  shorttitle = {{{AdaFed}}},
  author = {Hamidi, Shayan Mohajer and Yang, En-Hui},
  year = 2023,
  month = aug,
  journal = {Transactions on Machine Learning Research},
  issn = {2835-8856},
  urldate = {2026-08-16},
  langid = {english},
  
}

@misc{hamidiDistributedQuasiNewtonMethod2025,
  title = {Distributed {{Quasi-Newton Method}} for {{Fair}} and {{Fast Federated Learning}}},
  author = {Hamidi, Shayan Mohajer and Ye, Linfeng},
  year = 2025,
  month = jan,
  number = {arXiv:2501.10877},
  eprint = {2501.10877},
  primaryclass = {cs.LG},
  publisher = {arXiv},
  doi = {10.48550/arXiv.2501.10877},
  urldate = {2026-08-10},
  archiveprefix = {arXiv},
  
}

@article{huangFairnessAccuracyHorizontal2022,
  title = {Fairness and Accuracy in Horizontal Federated Learning},
  author = {Huang, Wei and Li, Tianrui and Wang, Dexian and Du, Shengdong and Zhang, Junbo and Huang, Tianqiang},
  year = 2022,
  month = apr,
  journal = {Information Sciences},
  volume = {589},
  pages = {170--185},
  issn = {0020-0255},
  doi = {10.1016/j.ins.2021.12.102},
  urldate = {2026-08-21},
  
}

@article{huFederatedLearningMeets2022,
  title = {Federated {{Learning Meets Multi-Objective Optimization}}},
  author = {Hu, Zeou and Shaloudegi, Kiarash and Zhang, Guojun and Yu, Yaoliang},
  year = 2022,
  month = jul,
  journal = {IEEE Transactions on Network Science and Engineering},
  volume = {9},
  number = {4},
  pages = {2039--2051},
  issn = {2327-4697},
  doi = {10.1109/TNSE.2022.3169117},
  urldate = {2026-08-16},
  
}

@article{kairouzAdvancesOpenProblems2021,
  title = {Advances and {{Open Problems}} in {{Federated Learning}}},
  author = {Kairouz, Peter and McMahan, H. Brendan and Avent, Brendan and Bellet, Aur{\'e}lien and Bennis, Mehdi and Nitin Bhagoji, Arjun and Bonawitz, Kallista and Charles, Zachary and Cormode, Graham and Cummings, Rachel and D'Oliveira, Rafael G. L. and Eichner, Hubert and El Rouayheb, Salim and Evans, David and Gardner, Josh and Garrett, Zachary and Gasc{\'o}n, Adri{\`a} and Ghazi, Badih and Gibbons, Phillip B. and Gruteser, Marco and Harchaoui, Zaid and He, Chaoyang and He, Lie and Huo, Zhouyuan and Hutchinson, Ben and Hsu, Justin and Jaggi, Martin and Javidi, Tara and Joshi, Gauri and Khodak, Mikhail and Konecn{\'y}, Jakub and Korolova, Aleksandra and Koushanfar, Farinaz and Koyejo, Sanmi and Lepoint, Tancr{\`e}de and Liu, Yang and Mittal, Prateek and Mohri, Mehryar and Nock, Richard and {\"O}zg{\"u}r, Ayfer and Pagh, Rasmus and Qi, Hang and Ramage, Daniel and Raskar, Ramesh and Raykova, Mariana and Song, Dawn and Song, Weikang and Stich, Sebastian U. and Sun, Ziteng and Suresh, Ananda Theertha and Tram{\`e}r, Florian and Vepakomma, Praneeth and Wang, Jianyu and Xiong, Li and Xu, Zheng and Yang, Qiang and Yu, Felix X. and Yu, Han and Zhao, Sen},
  year = 2021,
  month = jun,
  journal = {Foundations and Trends\textregistered{} in Machine Learning},
  volume = {14},
  number = {1-2},
  pages = {1--210},
  issn = {1935-8237},
  doi = {10.1561/2200000083},
  urldate = {2026-08-18},
  
}

@inproceedings{karimireddySCAFFOLDStochasticControlled2020,
  title = {{{SCAFFOLD}}: {{Stochastic Controlled Averaging}} for {{Federated Learning}}},
  shorttitle = {{{SCAFFOLD}}},
  booktitle = {Proceedings of the 37th {{International Conference}} on {{Machine Learning}}},
  author = {Karimireddy, Sai Praneeth and Kale, Satyen and Mohri, Mehryar and Reddi, Sashank and Stich, Sebastian and Suresh, Ananda Theertha},
  year = 2020,
  month = nov,
  pages = {5132--5143},
  publisher = {PMLR},
  issn = {2640-3498},
  urldate = {2026-08-18},
  langid = {english},
  
}

@inproceedings{liDittoFairRobust2021,
  title = {Ditto: {{Fair}} and {{Robust Federated Learning Through Personalization}}},
  shorttitle = {Ditto},
  booktitle = {Proceedings of the 38th {{International Conference}} on {{Machine Learning}}},
  author = {Li, Tian and Hu, Shengyuan and Beirami, Ahmad and Smith, Virginia},
  year = 2021,
  month = jul,
  pages = {6357--6368},
  publisher = {PMLR},
  issn = {2640-3498},
  urldate = {2026-08-18},
  langid = {english},
  
}

@inproceedings{liFairResourceAllocation2019,
  title = {Fair {{Resource Allocation}} in {{Federated Learning}}},
  booktitle = {International {{Conference}} on {{Learning Representations}}},
  author = {Li, Tian and Sanjabi, Maziar and Beirami, Ahmad and Smith, Virginia},
  year = 2019,
  month = sep,
  urldate = {2026-08-16},
  langid = {english},
  
}

@misc{liFederatedOptimizationHeterogeneous2020,
  title = {Federated {{Optimization}} in {{Heterogeneous Networks}}},
  author = {Li, Tian and Sahu, Anit Kumar and Zaheer, Manzil and Sanjabi, Maziar and Talwalkar, Ameet and Smith, Virginia},
  year = 2020,
  month = apr,
  number = {arXiv:1812.06127},
  eprint = {1812.06127},
  primaryclass = {cs.LG},
  publisher = {arXiv},
  doi = {10.48550/arXiv.1812.06127},
  urldate = {2026-08-18},
  archiveprefix = {arXiv},
  
}

@inproceedings{liTiltedEmpiricalRisk2020,
  title = {Tilted {{Empirical Risk Minimization}}},
  booktitle = {International {{Conference}} on {{Learning Representations}}},
  author = {Li, Tian and Beirami, Ahmad and Sanjabi, Maziar and Smith, Virginia},
  year = 2020,
  month = oct,
  urldate = {2026-08-16},
  langid = {english},
  
}

@inproceedings{mcmahanCommunicationEfficientLearningDeep2017,
  title = {Communication-{{Efficient Learning}} of {{Deep Networks}} from {{Decentralized Data}}},
  booktitle = {Proceedings of the 20th {{International Conference}} on {{Artificial Intelligence}} and {{Statistics}}},
  author = {McMahan, Brendan and Moore, Eider and Ramage, Daniel and Hampson, Seth and y Arcas, Blaise Aguera},
  year = 2017,
  month = apr,
  pages = {1273--1282},
  publisher = {PMLR},
  issn = {2640-3498},
  urldate = {2026-08-16},
  langid = {english},
  
}

@misc{mohriAgnosticFederatedLearning2019,
  title = {Agnostic {{Federated Learning}}},
  author = {Mohri, Mehryar and Sivek, Gary and Suresh, Ananda Theertha},
  year = 2019,
  month = feb,
  number = {arXiv:1902.00146},
  eprint = {1902.00146},
  primaryclass = {cs.LG},
  publisher = {arXiv},
  doi = {10.48550/arXiv.1902.00146},
  urldate = {2026-08-16},
  archiveprefix = {arXiv},
  
}

@article{panFedLFLayerWiseFair2024,
  title = {{{FedLF}}: {{Layer-Wise Fair Federated Learning}}},
  shorttitle = {{{FedLF}}},
  author = {Pan, Zibin and Li, Chi and Yu, Fangchen and Wang, Shuyi and Wang, Haijin and Tang, Xiaoying and Zhao, Junhua},
  year = 2024,
  month = mar,
  journal = {Proceedings of the AAAI Conference on Artificial Intelligence},
  volume = {38},
  number = {13},
  pages = {14527--14535},
  issn = {2374-3468},
  doi = {10.1609/aaai.v38i13.29368},
  urldate = {2026-08-10},
  copyright = {Copyright (c) 2024 Association for the Advancement of Artificial Intelligence},
  langid = {english},
  
}

@article{panFedMDFGFederatedLearning2023,
  title = {{{FedMDFG}}: {{Federated Learning}} with {{Multi-Gradient Descent}} and {{Fair Guidance}}},
  shorttitle = {{{FedMDFG}}},
  author = {Pan, Zibin and Wang, Shuyi and Li, Chi and Wang, Haijin and Tang, Xiaoying and Zhao, Junhua},
  year = 2023,
  month = jun,
  journal = {Proceedings of the AAAI Conference on Artificial Intelligence},
  volume = {37},
  number = {8},
  pages = {9364--9371},
  issn = {2374-3468},
  doi = {10.1609/aaai.v37i8.26122},
  urldate = {2026-08-16},
  copyright = {Copyright (c) 2023 Association for the Advancement of Artificial Intelligence},
  langid = {english},
  
}

@inproceedings{reddiAdaptiveFederatedOptimization2020,
  title = {Adaptive {{Federated Optimization}}},
  booktitle = {International {{Conference}} on {{Learning Representations}}},
  author = {Reddi, Sashank J. and Charles, Zachary and Zaheer, Manzil and Garrett, Zachary and Rush, Keith and Kone{\v c}n{\'y}, Jakub and Kumar, Sanjiv and McMahan, Hugh Brendan},
  year = 2020,
  month = oct,
  urldate = {2026-08-29},
  langid = {english},
  
}

@article{shenMultiobjectiveFederatedLearning2025,
  title = {Multi-Objective Federated Learning: {{Balancing}} Global Performance and Individual Fairness},
  shorttitle = {Multi-Objective Federated Learning},
  author = {Shen, Yuhao and Xi, Wei and Cai, Yunyun and Fan, Yuwei and Yang, He and Zhao, Jizhong},
  year = 2025,
  month = jan,
  journal = {Future Generation Computer Systems},
  volume = {162},
  pages = {107468},
  issn = {0167-739X},
  doi = {10.1016/j.future.2024.07.046},
  urldate = {2026-08-18},
  
}

@article{shiFairnessAwareFederatedLearning2024,
  title = {Towards {{Fairness-Aware Federated Learning}}},
  author = {Shi, Yuxin and Yu, Han and Leung, Cyril},
  year = 2024,
  month = sep,
  journal = {IEEE Transactions on Neural Networks and Learning Systems},
  volume = {35},
  number = {9},
  pages = {11922--11938},
  issn = {2162-2388},
  doi = {10.1109/TNNLS.2023.3263594},
  urldate = {2026-08-16},
}

@misc{varamNotAllObjectives2026,
  title = {Not {{All Objectives Are Born Equal}}: {{Priority-Constrained Descent}} for {{Hierarchical Multi-Objective Optimization}}},
  shorttitle = {Not {{All Objectives Are Born Equal}}},
  author = {Varam, Dara and Alhajri, Mohamed I.},
  year = 2026,
  month = jun,
  number = {arXiv:2606.29521},
  eprint = {2606.29521},
  primaryclass = {cs.LG},
  publisher = {arXiv},
  doi = {10.48550/arXiv.2606.29521},
  urldate = {2026-08-21},
  archiveprefix = {arXiv},
}

@misc{wangCRAFTConflictResolvedAggregation2026,
  title = {{{CRAFT}}: {{Conflict-Resolved Aggregation}} for {{Federated Training}}},
  shorttitle = {{{CRAFT}}},
  author = {Wang, Ziqi and Liu, Qiang and Thuerey, Nils},
  year = 2026,
  month = may,
  number = {arXiv:2605.21317},
  eprint = {2605.21317},
  primaryclass = {cs.LG},
  publisher = {arXiv},
  doi = {10.48550/arXiv.2605.21317},
  urldate = {2026-08-16},
  archiveprefix = {arXiv},
}

@misc{wangFederatedLearningFair2021,
  title = {Federated {{Learning}} with {{Fair Averaging}}},
  author = {Wang, Zheng and Fan, Xiaoliang and Qi, Jianzhong and Wen, Chenglu and Wang, Cheng and Yu, Rongshan},
  year = 2021,
  month = jun,
  number = {arXiv:2104.14937},
  eprint = {2104.14937},
  primaryclass = {cs.LG},
  publisher = {arXiv},
  doi = {10.48550/arXiv.2104.14937},
  urldate = {2026-08-18},
  archiveprefix = {arXiv},
}

@misc{xiaoFashionMNISTNovelImage2017,
  title = {Fashion-{{MNIST}}: A {{Novel Image Dataset}} for {{Benchmarking Machine Learning Algorithms}}},
  shorttitle = {Fashion-{{MNIST}}},
  author = {Xiao, Han and Rasul, Kashif and Vollgraf, Roland},
  year = 2017,
  month = sep,
  number = {arXiv:1708.07747},
  eprint = {1708.07747},
  primaryclass = {cs.LG},
  publisher = {arXiv},
  doi = {10.48550/arXiv.1708.07747},
  urldate = {2026-09-03},
  archiveprefix = {arXiv},
}

@article{krizhevsky2009learning,
  title = {Learning Multiple Layers of Features from Tiny Images},
  author = {Krizhevsky, Alex and Hinton, Geoffrey and others},
  year = 2009,
  publisher = {Toronto, ON, Canada}
}

@inproceedings{acarFederatedLearningBased2020,
  title = {Federated {{Learning Based}} on {{Dynamic Regularization}}},
  booktitle = {International {{Conference}} on {{Learning Representations}}},
  author = {Acar, Durmus Alp Emre and Zhao, Yue and Matas, Ramon and Mattina, Matthew and Whatmough, Paul and Saligrama, Venkatesh},
  year = 2020,
  month = oct,
  urldate = {2026-09-04},
  langid = {english},
}

@article{jain1984quantitative,
  title={A quantitative measure of fairness and discrimination},
  author={Jain, Rajendra K and Chiu, Dah-Ming W and Hawe, William R and others},
  journal={Eastern Research Laboratory, Digital Equipment Corporation, Hudson, MA},
  volume={21},
  number={1},
  pages={2022--2023},
  year={1984}
}
